%% file: main_arxiv.tex
\def\year{2022}\relax
\documentclass[letterpaper]{article} 
\usepackage{arxiv}  
\usepackage{times}  
\usepackage{helvet}  
\usepackage{courier}  
\usepackage[hyphens]{url}  
\usepackage{graphicx} 
\usepackage{natbib}  
\usepackage{caption} 
\DeclareCaptionStyle{ruled}{labelfont=normalfont,labelsep=colon,strut=off} 
\usepackage{algorithm}
\usepackage{algorithmic}
\usepackage{subcaption}
\usepackage{newfloat}
\usepackage{listings}
\floatstyle{ruled}
\newfloat{listing}{tb}{lst}{}
\floatname{listing}{Listing}

\title{Programmatic Reward Design by Example}
\author{
    Weichao Zhou and Wenchao Li 
}
\affiliations{
    Boston University\\
    \{zwc662,wenchao\}@bu.edu
}

\usepackage{xcolor}
\usepackage{mathtools}
\usepackage{amsfonts}
\usepackage{amssymb}
\usepackage{amsthm}
\usepackage{amsmath}
\graphicspath{{figures/}}

\newtheorem{theorem}{Theorem}
\newtheorem{definition}{Definition}

\newtheorem{example}{Example}
\usepackage[normalem]{ulem}
\usepackage{xcolor}

\newcommand{\li}[1]{}

\newcommand{\zwccross}[1]{}
\newcommand{\zwcadd}[1]{#1}
\newcommand{\zwcchange}[2]{\zwcadd{#2}\zwccross{#1}}

\newcommand{\camadd}[1]{#1}

\usepackage{bibentry}

\begin{document}

\maketitle

\input{arxiv/abs}
\input{arxiv/intr}
\input{arxiv/rela}
\input{arxiv/sec1}
\input{arxiv/sec2}
\input{arxiv/sec3}

\input{arxiv/sec4}

\input{arxiv/sec5}
\input{arxiv/concl}

\bibliography{main_arxiv}
\input{arxiv/sec6}

\end{document}

%% file: arxiv/abs.tex
\begin{abstract}
\begin{quote}
    {{Reward design is a fundamental problem in reinforcement learning (RL). A misspecified or poorly designed reward can result in low sample efficiency and undesired behaviors.
    In this paper, we propose the idea of \textit{programmatic reward design}, i.e. using programs to specify the reward functions in RL environments. 
    Programs allow human engineers to express sub-goals and complex task scenarios in a structured and interpretable way.
    The challenge of programmatic reward design, however, is that while humans can provide the high-level structures, properly setting the low-level details, such as the right amount of reward for a specific sub-task, remains difficult. 
    A major contribution of this paper is a probabilistic framework that can infer the best candidate programmatic reward function from expert demonstrations. Inspired by recent generative-adversarial approaches, our framework \zwcchange{trains an agent policy to generate trajectories while searching for the most likely program that can discriminate the demonstrated trajectories from those of the agent}{searches for the most likely programmatic reward function under which the optimally generated trajectories cannot be differentiated from the demonstrated trajectories}. Experimental results show that programmatic reward functions learned using this framework can significantly outperform those learned using existing reward learning algorithms, and enable RL agents to achieve state-of-the-art performance on highly complex tasks.
    }}
\end{quote}
\end{abstract}

%% file: arxiv/intr.tex
\section{Introduction}\label{intro}

Reward signals are an integral part of reinforcement learning (RL). 
Most conventional reward functions are goal-driven -- they reward the agent only at the end of each episode.
The sparsity of such reward signals, however, can make RL algorithms highly inefficient.
As the complexity of the task increases, it becomes difficult for the agent to grasp the intricacies of the task solely from a goal-driven reward \cite{DBLP:journals/corr/AmodeiOSCSM16}. \zwccross{A naive reward mapping can also give rise to various negative effects in RL }
\li{not clear what naive reward mapping means here}
\li{what about other reward shaping approaches?}

Inverse Reinforcement Learning (IRL) is a general paradigm that aims at recovering the intrinsic reward function of human experts from their demonstrations~\cite{Ng:2000:AIR:645529.657801,Ziebart:2008:MEI:1620270.1620297}. Earlier works of IRL require the provision of multiple feature functions to construct the reward function. 
More recent attempts use function approximation by means of deep neural networks to alleviate this limitation and have considerable success~\cite{fu2017learning,DBLP:journals/corr/FinnLA16,DBLP:journals/corr/FinnCAL16}. 
However, due to the lack of interpretability of the approximated reward functions, it is difficult to enforce specific correctness constraints in the reward learning process.\zwccross{and to verify whether a learned reward function satisfies certain desired properties}\li{why is verification difficult?}

Recent works proposed logic-based reward functions~\cite{li2017reinforcement,ijcai2019-840} \zwcchange{as competitive solutions for endowing}{to endow} an RL agent with high-level knowledge of the task via logical specifications. \zwccross{The reward func works focus on defining a logical specification and fulfilling a task specification with a reward function that is usually characterized as an automaton translated from the specification.}\zwcadd{The logic-constrained reward \cite{hasanbeig2019logically} and reward machines \cite{pmlr-v80-icarte18a} explicitly represents reward functions as automata. 
However, it is still cumbersome to design the automata and they can be difficult to understand when the size of the automata are large. 
}
\li{any shortcoming of this type of approach? needs better transition here}In this paper, we propose \textit{programmatic reward functions}, i.e. reward functions represented as programs expressed in human-readable domain specific language (DSL). 
There are several benefits of using programmatic reward functions for RL policy training. 
First, programs allow human engineers to express sub-goals, complex task scenarios in a structural and interpretable way.
The inclusion of such domain knowledge forms inductive biases that help improve the sample efficiency and performance of RL agents.
Second, engineers can take advantage of the rich semantics of DSLs to explicitly memorize, manipulate and leverage hindsight experiences of the RL agent. 
Lastly, programs are amenable to the specification of symbolic constraints \zwcadd{over the holes}. \li{may need to clarify constraints over what} 
In a typical design routine of a programmatic reward function, 
we assume that an engineer can provide human insights in the form of a partial program, or a \textit{sketch} \cite{solar2008program}\zwcadd{, analogous to human providing feature functions in IRL and logical specifications in logic-based reward designs,} to express the high-level structures of the task specification or reward function. 
The \textit{sketch} in essence defines the set of events and certain interactions among them that the human engineer deem relevant to the task. 
The low-level details, such as the right amount of reward for a specific event or sub-task, are left as \textit{holes}.

Similar to Programming by Example (PBE)~\cite{menon2013ml-for-pbe}, we propose to infer the holes in a programmatic reward sketch from expert demonstrated trajectories.
A key difference of our approach from PBE is that the demonstrated trajectories do not directly correspond to inputs or (intermediate) outputs of the program, but instead are assumed to be generated by an unknown expert policy that is optimal under some realization of the programmatic reward function.
A major contribution of this paper is a probabilistic learning algorithm that can complete a given programmatic reward sketch based on expert demonstrations.
Our overall framework, called \textit{Programmatic Reward Design by Example (PRDBE)}, consists of three components: a set of \textit{example trajectories} demonstrated by a human expert, a \textit{program sketch} and a \textit{symbolic constraint} that the complete program should satisfy. 
Directly searching in the program space is a combinatorial problem and can easily become intractable. 
Our approach is to search for the \textit{most likely} program that matches the expert's intrinsic reward function. 
Our solution is inspired by generative adversarial approaches~\cite{DBLP:journals/corr/FinnCAL16,NEURIPS2018_943aa0fc} which introduce a discriminator to distinguish agent trajectories from the expert's. 
However, instead of \zwcchange{using}{formulating}\li{formulating?} an agent's policy as a generative model, we sample trajectories that are optimal under a candidate programmatic reward function and \zwcadd{iteratively} \li{iteratively?} \zwcadd{improve the candidate program to maximize the chance of the discriminator making}\li{any notion of maximization here?} false predictions on those trajectories. 
To circumvent the issue of non-differentiability of programs, we employ a sampler to sample candidate programs \zwcchange{from a given search space}{from the space of valid programs}\li{not clear how the search space is given; or can we just say from the space of valid programs?}. 
In particular, we use self-normalized importance sampling to sample trajectories from an agent's policy. 
We summarize our contributions below.
\begin{itemize}
\item We propose Programmatic Reward Design by Example \zwcadd{(PRDBE)}, a novel paradigm to design and learn program-like reward functions for RL problems.
\item We develop a probabilistic learning framework that can infer the most likely candidate reward program from expert demonstrations. 
\item Our approach enables RL agents to achieve state-of-the-art performance on highly complex environments with only a few demonstrations. In addition, we show that programmatic reward functions generalize across different environment configurations of the same task. 
\end{itemize}



%% file: arxiv/rela.tex
\section{Related Work}

\li{I still need to go over this section again}
\noindent\textbf{Inverse Reinforcement Learning}. 
IRL \cite{Ng:2000:AIR:645529.657801,Abbeel:2004:ALV:1015330.1015430} instantiates a learning-from-demonstrations (LfD) framework to effectively infer the intrinsic reward functions of human experts. It is also notable for having infinite number of solutions. 
To resolve the ambiguity of IRL, max-entropy \cite{Ziebart:2008:MEI:1620270.1620297}, max-margin \cite{Abbeel:2004:ALV:1015330.1015430,Ratliff:2006:MMP:1143844.1143936} and Bayesian \cite{ramachandran2007bayesian} methods have been proposed. However, those IRL methods assume linear rewards on the basis of non-linear state features, 
and also call RL in a loop to repeatedly solve the entire environments. The recent works \cite{fu2017learning,ho2016generative,NEURIPS2018_943aa0fc,DBLP:journals/corr/FinnCAL16}, drawing a connection between IRL and Generative Adversarial Networks (GANs) \cite{goodfellow2014generative}, have achieved substantially improved the scalablity of IRL by using deep neural networks and data-driven approaches. Our work, while embracing a data-driven ideology, represents the reward function in an interpretable way.

\noindent\textbf{Reward Design}. \zwcadd{Reward shaping \cite{Ng:1999:PIU:645528.657613} highly speeds up RL training by modifying a sparse reward functions with state-based potential functions. Intrinsic reward generation \zwcchange{with the state-visitation count-based \cite{bellemare2016unifying}, curiosity-driven \cite{pathak2017curiosity}, impact-driven \cite{DBLP:journals/corr/abs-1708-08611}}{\cite{bellemare2016unifying,pathak2017curiosity,DBLP:journals/corr/abs-1708-08611}} and adversarially-guided  \cite{flet-berliac2021adversarially} techniques aims at motivating the agent to exhaustively explore the environments.} 
\li{We don't need to include generic RL works. The curiosity-based approaches are relevant. Is there a better categorization? see my later comment on reward shaping as well.} \zwcadd{Unlike these approaches, our work aims at capturing human knowledge in the reward function and 
does not generate uninterpretable reward signals densely ranging over the entire state space.} Logic-based reward designs \cite{li2017reinforcement,hasanbeig2019logically,ijcai2019-840} present human knowledge in reward functions with specification languages such as linear temporal logic (LTL)\cite{10.5555/1373322}. Reward machine theories \cite{DBLP:journals/corr/abs-2010-03950} further directly represent the reward functions as finite state automata which can be translated into logic formulas. \zwccross{Nonetheless, many logic languages have limited expressiveness and flexibility in specifying complex human insights.}Our work distinguishes itself by 1) using programming languages instead of logics to expressively represent human insights in the reward functions; 2) adopting LfD to implement low-level details in the reward functions. \zwcadd{Regarding LfD, }Inverse reward design (IRD) \cite{NIPS2017_32fdab65} design reward functions in a manner \zwccross{of LfD}similar to IRL. Safety-aware apprenticeship learning \cite{zhou2018safety} incorporate formal specification and formal verification with IRL. \zwcadd{However, those works restrict the reward function to be a linear combination of state features. Our work does not have such limitations.}
\li{what about reward shaping?}

\noindent\textbf{Interpretable Reinforcement Learning}. 
Learning interpretable RL policies has drawn continual attention \cite{andre2001programmable,andre2002state,verma2018programmatically,DBLP:journals/corr/abs-1907-07273,DBLP:journals/corr/abs-2102-11137,tian2020learning}. Our work focuses on programmatic reward functions instead of policies \zwcadd{because well-designed reward functions can benefit diverse RL training methods.} There have been a variety of works on learning-based program synthesis \cite{ellis2021dreamcoder,parisotto2016neuro,itzhaky2010a,programming-examples-pl-meets-ml,NIPS2015_5785}.
Our work is inspired by the concept of sketch synthesis \cite{solar2008program}. Realizing sketch synthesis from example for reward function is our main contribution in this paper.  

%% file: arxiv/sec1.tex
\section{Background}\label{prelim}

\li{italicize certain recurring terms, such as trajectory}
\textbf{Reinforcement Learning.} RL problems model an environment as a Markov Decision Process (MDP) $\mathcal{M}:=\langle \mathcal{S},\mathcal{A}, \mathcal{P}, d_0\rangle$ where $\mathcal{S}$ is a state space, $\mathcal{A}$ is an action space, $\mathcal{P}(s'|s, a)$ is the probability of reaching a state $s'$ by performing an action $a$ at a state $s$, $d_0$ is an initial state distribution. A \textit{trajectory} $\tau=s^{(0)}a^{(0)}s^{(1)}a^{(1)}\ldots s^{(T)}a^{(T)}$ is produced by sequentially performing actions for $T$ steps after starting from an initial state $s^{(0)}\sim d_0$. An RL agent can select actions according to a control \textit{policy} $\pi(a|s)$ which determines the probability of performing action $a$ at any state $s$.  A \textit{reward function} is a mapping $f:\mathcal{S}\times \mathcal{A}\rightarrow \mathbb{R}$ from state-action pairs to the real space. With a slight abuse of notations, we also represent the total reward and the joint probability along a trajectory $\tau$ as $f(\tau)=\sum^{T}_{t=0} f(s^{(t)}, a^{(t)})$ and  $\pi(\tau)=\prod^{T}_{t=0}\pi(a^{(t)}|s^{(t)})$ respectively.  
\zwccross{In the probabilistic graphical model (PGM) of RL, a binary random variable $o^{(t)}$ is used to indicate whether the behavior at time step $t$ is optimal \cite{levine2018reinforcement}. The probability of a state-action pair being optimal is $p(o^{(t)}=1|s^{(t)}, a^{(t)};f)\propto exp(f(s^{(t)},a^{(t)}))$. The objective of RL is then to minimize $J_{RL}(\pi)=D_{KL}\big(\pi(\tau)||p(\tau)\prod^T_{t=0} p(a^{(t)}|s^{(t)},o^{(t)}=1)\big)$ where $p(\tau)=d_0(s^{(0)})\prod^{T-1}_{t=0}\mathcal{P}(s^{(t+1)}|s^{(t)},a^{(t)})$ is the distribution of $\tau$ under the passive dynamics of the environment.}{Let $H(\pi)$ be the entropy of $\pi$. The objective of entropy-regularized RL~\cite{levine2018reinforcement} is to minimize $J_{RL}(\pi)=\mathbb{E}_{\tau\sim \pi}[f(\tau)] - H(\pi)$.}


\noindent
\textbf{Learning from Demonstrations.} When the reward function is not given but a set $E$ of expert trajectories $\tau_E$'s is demonstrated by some unknown expert policy $\pi_E$,  GAIL\cite{ho2016generative} trains an agent policy $\pi_A$ to imitate $\pi_E$ via a generative adversarial objective as in \eqref{eq1_1} where a discriminator $D:\mathcal{S}\times \mathcal{A}\rightarrow[0, 1]$ is trained to maximize \eqref{eq1_1} so that it can identify whether an $(s,a)$ is sampled from $\tau_E$'s; $\pi_A$ as the generator is optimized to minimize \eqref{eq1_1} so that its trajectories $\tau_A$'s are indistinguishable from $\tau_E$'s.  
Bayesian GAIL \cite{NEURIPS2018_943aa0fc} labels any expert trajectory $\tau_E$ with $1_E$ and $0_E$ to respectively indicate classifying $\tau_E$ as being sampled from $E$ or not. Likewise, $1_A$ and $0_A$ indicate classifying any agent trajectory $\tau_A$ as from $E$ or not. Bayesian GAIL learns the most likely $D$ that makes the correct classifications $0_A, 1_E$ in terms of $p(D|0_A, 1_E; \pi_A; E)\propto p(D)p(0_A, 1_E|\pi_A, D; E)\propto\sum_{\tau_A} p(\tau_A|\pi_A)p(0_A|\tau_A; D)\sum_{\tau_E} p(\tau_E|E) p(1_E|\tau_E;D)$ where $p(1|\tau; D)=\prod^T_{t=0}D(s^{(t)},a^{(t)})$. Its logarithm \eqref{eq1_0} is lower-bounded by \eqref{eq1_1} due to Jensen's inequality. It is shown in \cite{fu2017learning} that by representing $D(s,a):=\frac{\exp(f(s,a))}{\exp(f(s,a)) + \pi_A(a|s)}$ with a neural network $f$, \eqref{eq1_1} is maximized when $f\equiv \log \pi_E$, implying that $f$ is a reward function under which $\pi_E$ is optimal. Hence, by representing $D$ in \eqref{eq1_0} and \eqref{eq1_1} with $f$, an IRL objective of solving the most likely expert reward function $f$ is obtained.
\begin{eqnarray}
&\log \sum\limits_{\tau_E} p(\tau_E|E) p(1_E|\tau_E;D)\sum\limits_{\tau_A} p(\tau_A|\pi_A)p(0_A|\tau_A; D)& \label{eq1_0}\\
&\geq\qquad \underset{\mathclap{\tau_E\sim E}}{\mathbb{E}}\quad\ \Big[\log \prod\limits^T_{t=0}D(s_E^{(t)},a^{(t)}_E)\Big ]+\qquad\qquad&\nonumber\\
&\qquad\underset{\mathclap{\tau_A\sim\pi_A}}{\mathbb{E}}\quad\ \Big[\log \prod\limits^T_{t=0}1-D(s_A^{(t)},a^{(t)}_A)\Big]&\label{eq1_1}
\end{eqnarray}

\noindent    
\textbf{Program Synthesis by Sketching.} 
In the sketch synthesis problem, a human designer provides a \textit{sketch} $\mathtt{e}$, i.e., an incomplete program wherein certain details are left empty. Each unknown detail is called a \textit{hole} and denoted as $\mathtt{?_{id}}$, indexed by $\mathtt{id}$ in order of its appearance in $\mathtt{e}$. The formalism of the sketches follows a general grammar $\Lambda$ as in \eqref{syn3_0} and \eqref{syn3_1} where $\mathtt{\mathcal{G}}$ is a family of functions customized with domain-specific knowledge for the tasks; $\mathtt{x}$ represents the input argument of the program. The grammar $\mathtt{\Lambda}$ induces a set $\mathtt{\mathcal{E}}$ of syntactically allowable sketches. When a sketch $\mathtt{e\in\mathcal{E}}$ is given by the designer,  $\mathtt{\textbf{?}_e=\{?_1,?_2,\ldots\}}$ denotes an ordered set of the holes appearing in $\mathtt{e}$. Let $\mathtt{\mathcal{H}}$ be a set of possible assignments to $\mathtt{\textbf{?}_e}$. The sketch $\mathtt{e}$ and $\mathtt{\mathcal{H}}$ induce a set ${\mathcal{L}=\{l\triangleq \mathtt{e[\textbf{h}/\textbf{?}_e]| \textbf{h}\in \mathcal{H}}}\}$ of complete programs where $\mathtt{e[\textbf{h}/\textbf{?}_e]}$ means substituting $\mathtt{\textbf{?}_e}$ in $\mathtt{e}$ with an assignment $\mathtt{\textbf{h}}$. Besides the syntax, the program $l$ is required to be a function with a list type $[(\mathcal{S}, \mathcal{A})]$ argument. A valid assignment to the argument can be a list-represented trajectory, i.e., by writing $\tau=s^{(0)}a^{(0)}\ldots s^{(t)}a^{(t)}$ as $[(s^{(0)}a^{(0)}),\ldots, (s^{(t)},a^{(t)})]$. Hereinafter we refer to $\tau$ either as a sequence of state-action pairs or as a list of state-action tuples depending on the contexts. \camadd{The rules in \eqref{sem3_0} and \eqref{sem3_1} define the semantics of $\mathtt{\Lambda}$ wherein \eqref{sem3_1} replaces $\mathtt{x}$ with $\tau$ in $\mathtt{g(e_1,\ldots,e_n)}$. The result of directly applying input $\tau$ to a sketch $\mathtt{e}$ is written as $\mathtt{[\![e]\!]}(\tau)$ which induces a partial program with $\mathtt{\textbf{?}_e}$ as free variables. For any complete program $l$, given a trajectory input $\tau$, the output ${[\![l]\!]}(\tau)$ is required to be a $|\tau|$-length real-valued list.} 
 \begin{eqnarray}
Sketch\ \mathtt{e} &:=& \mathtt{u\ |\ g(e_1, \ldots, e_n)\qquad g\in \mathcal{G}}\label{syn3_0}\\
Term\ \mathtt{u}&:=& \mathtt{const\ |\ \mathtt{?_{id}}\ |\ x }\label{syn3_1}
\end{eqnarray}
\begin{eqnarray}
&\mathtt{[\![const]\!]}(\tau) := \mathtt{const\ \ [\![\mathtt{?_{id}}]\!](\tau):=\mathtt{?_{id}}\ \  [\![x]\!](\tau):=\tau} &\ \label{sem3_0}\\
&\mathtt{[\![g(e_1,\ldots, e_n)]\!]}(\tau) := \mathtt{g(e_1, \ldots, e_n)[ \tau/x]}&\label{sem3_1}
\end{eqnarray}

%% file: arxiv/sec2.tex
\section{Motivating Example}
In this section, we motivate the problem of programmatic reward design \camadd{with the pseudo-code of two programmatic reward function sketches, one for {a navigation task in a gridworld} and the other one for representing a reward function that has been formulated as a finite state automaton (FSA). 
For those two tasks, we assume that the domain expert provides a set of tokens such as $\mathtt{reach\underline{\ }goal, unlock\underline{\ }door}$ representing the behavior of the agent, and $\mathtt{A,B,C}$ representing the FSA states. A predicate function $\mathtt{pred(\cdot)}\in \mathcal{G}$ of trajectory can output those tokens according to the last state-action in the input trajectory.}
 
\begin{figure}
     \centering
     \begin{subfigure}[b]{0.1\textwidth}
         \centering
        \includegraphics[height=2.7cm, width=1.8cm]{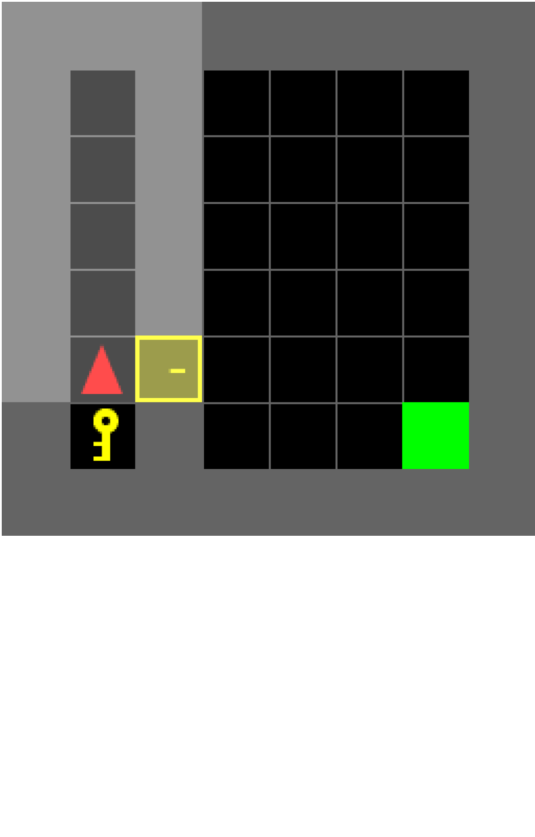}
         \caption{}
         \label{fig1_1}
     \end{subfigure}
     \begin{subfigure}[b]{0.32\textwidth}
         \centering
         \includegraphics[height=3.3cm,width=7.3cm]{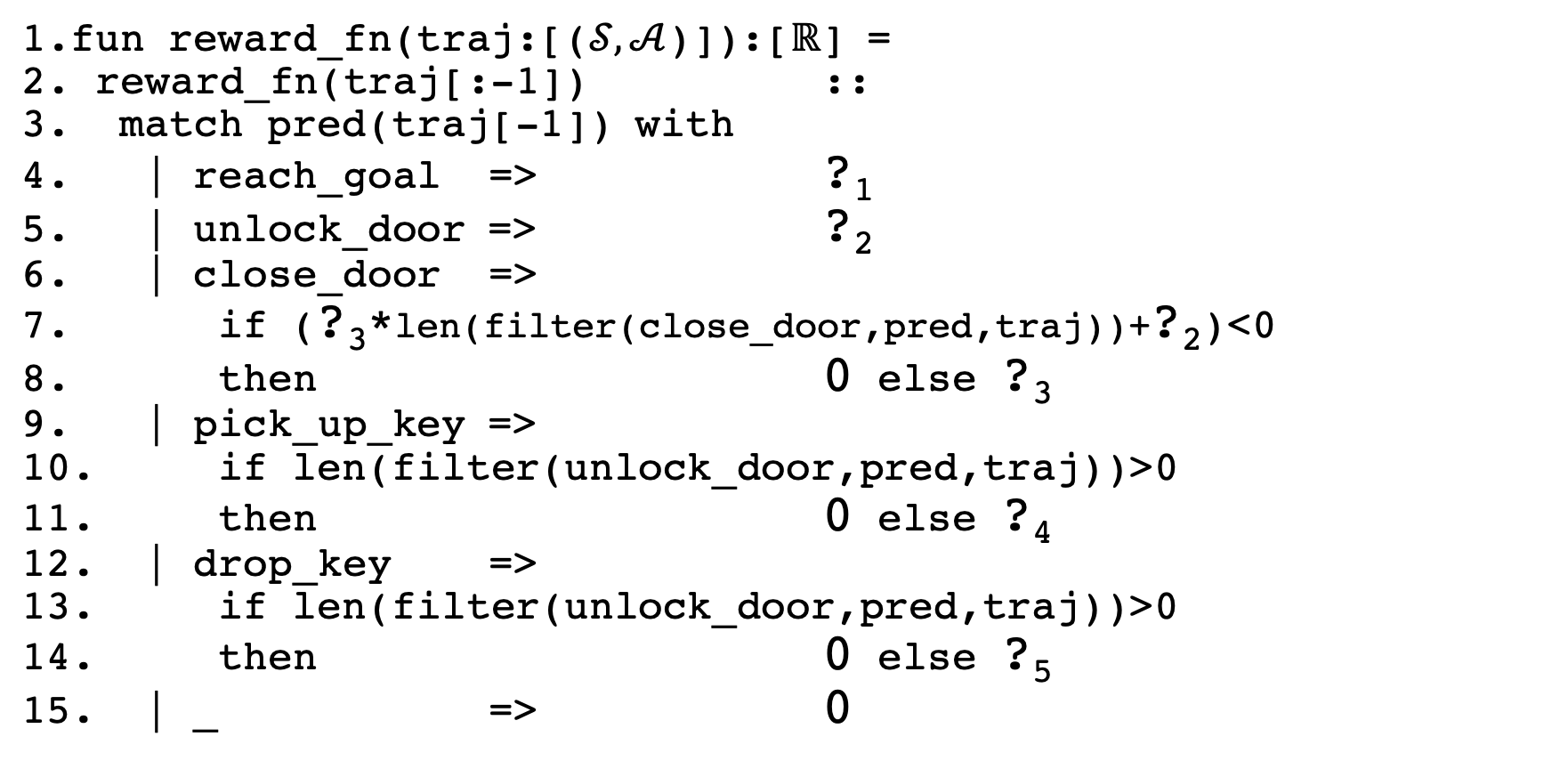}
         \caption{}
         \label{fig1_2}
     \end{subfigure}
     \begin{subfigure}[b]{0.1\textwidth}
         \centering
        \includegraphics[height=3.1cm, width=1.9cm]{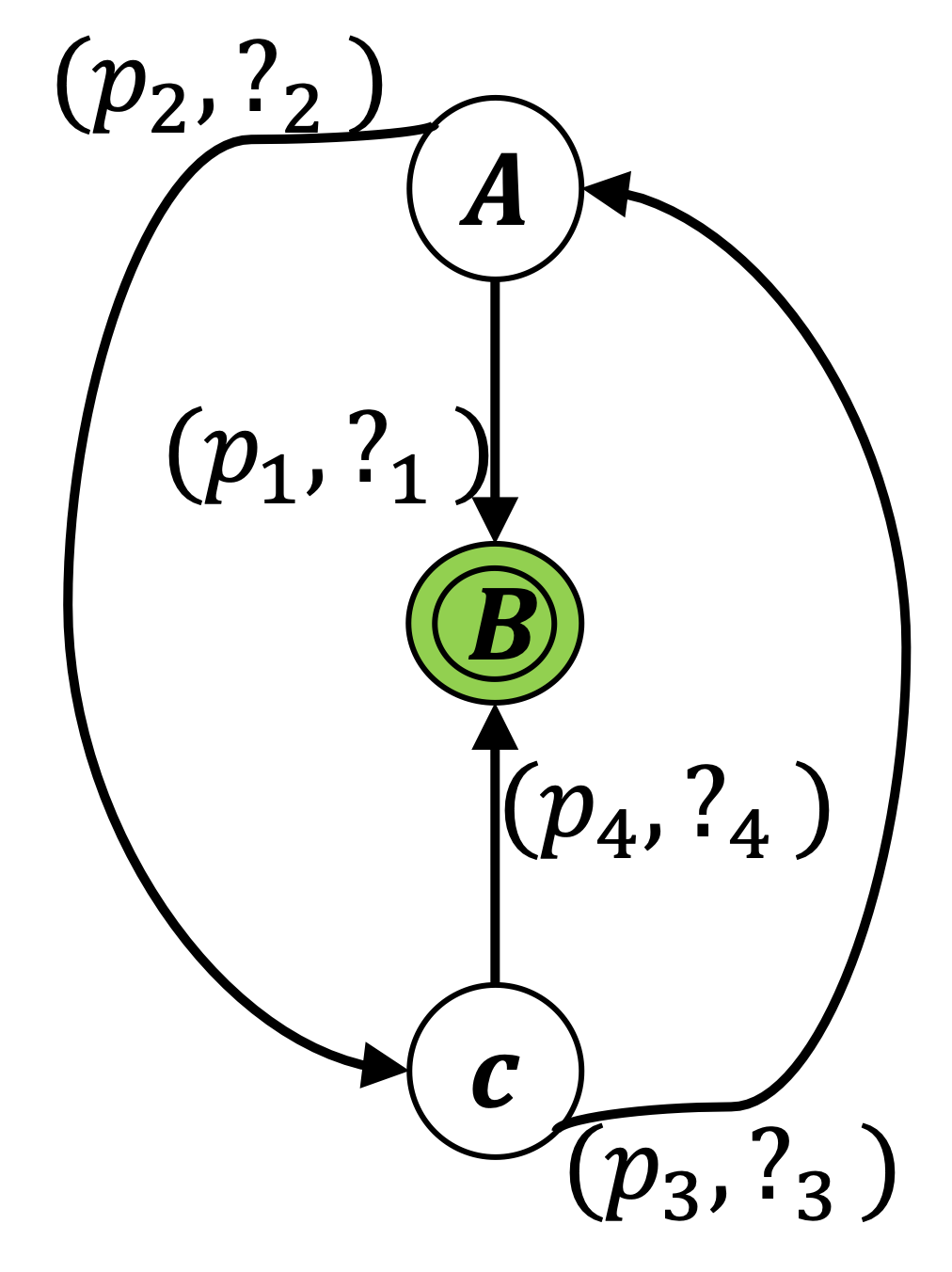}
        \caption{} 
         \label{fig3_1}
     \end{subfigure}
        \begin{subfigure}[b]{0.32\textwidth}
         \centering
         \includegraphics[height=3.2cm,width=6.4cm]{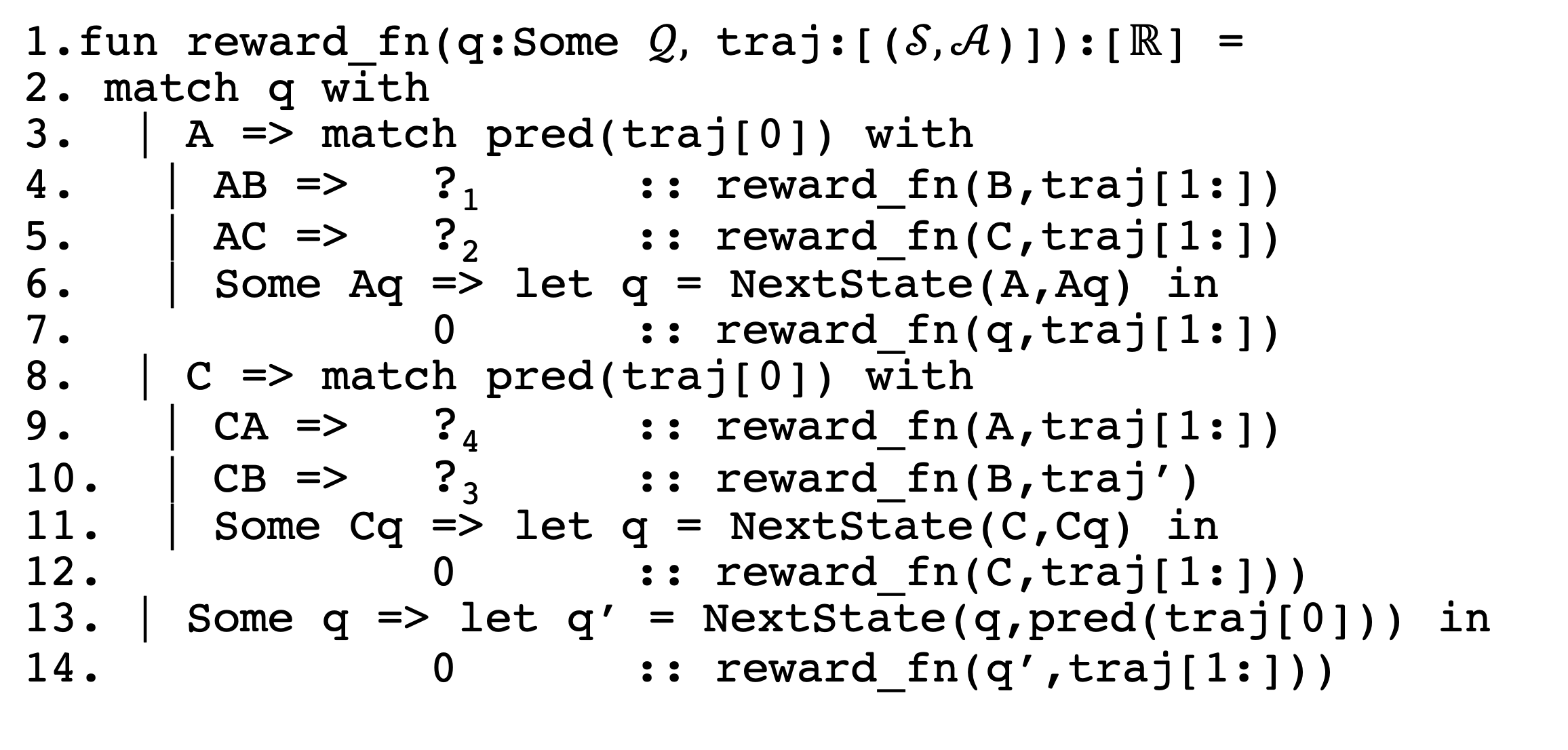}
         \caption{}
         \label{fig3_2}
     \end{subfigure}
     \caption{(a) MiniGrid 8x8 Door-Key; (b) A programmatic reward sketch for (a);(c) A reward function as a finite state automaton; (d) a programmatic reward sketch for (c)} 
\end{figure}
\begin{example}\textbf{(Door-Key Task)}\label{eg2_1} Fig.\ref{fig1_1} is an $8\times8$ Door-Key task in a Mini-Grid environment~\cite{gym_minigrid}. An agent needs to pick up a key, unlock the yellow door on the grey wall and reach the green goal tile. In every step, the agent can observe at most $7\times 7$ area in front if the area is not blocked by walls and doors. By default, the environment only returns a reward when the agent reaches the goal tile.
\end{example}
A sketch $\mathtt{reward\underline{\ } fn}$ of programmatic reward function for Example \ref{eg2_1} is shown in Figure\ref{fig1_2}. The holes $\mathtt{\{{?_{id}}\}^{5}_{id=1}}$ are unknown numerical values. The input argument is written as $\mathtt{traj}$ instead of $\mathtt{x}$ for readability. The output is a real-valued list concatenating the results of each recursive call to $\mathtt{reward\underline{\ } fn}$. The elements in the output list are the very rewards for the state-action pairs in the input trajectory in successive steps. This sketch responds to events such as reaching the goal and unlocking the door. 
We note that this programmatic reward function, whereas sparse in appearance considering the large state space, differs from the default goal-driven reward by informing the agent the stage-wise completion of the task.  Line $7$, $10$ and $13$ explicitly leverage hindsight experience to determine reward for the current step. 
Critically, $\mathtt{?_{id}}$'s in Figure\ref{fig1_2} ought not to be arbitrarily assigned. Suppose that a penalty $\mathtt{?_5}\leq 0$ for dropping the key is less than a award $\mathtt{?_4}\geq 0$ for picking up the key. The agent would repeatedly pick up and drop the key to obtain high net gain $\mathtt{?_5+?_4}\geq 0$ instead of reaching for the goal state. Besides, we observe that the penalties for redundant actions such as closing the door ought not to be overwhelmingly high. Otherwise an under-trained agent may be incentivized to reside away from the door for good. 

\begin{example}\textbf{(Reward Function as A Finite State Automaton (FSA))} \label{eg3_1} In Fig.\ref{fig3_1}, a reward function for some RL problem is represented as an FSA comprised of at least $3$ states $A, B, C$ among which $A$ indicates task initialization, $B$ indicates task completion, $C$ indicates the occurrence of some other event of interests. Each directed edge represents a state-transition triggered by a step in the environment.  Other states and transitions are omitted. Each transition is annotated with $(p_{\mathtt{id}},\mathtt{?_{id}})$ where $p_{\mathtt{id}}$ is an unknown transition probability dependent on both the environment and the RL policy; $\mathtt{?_{id}}$ is the reward sent to the agent at the moment of the transition. An RL agent is supposed to accomplish the task with minimal amount of steps in the environment. Note that the states in Fig.\ref{fig3_1} are not to be confused with the MDP states in the environment; one step in the environment does not necessarily trigger the transitions drawn in Fig.\ref{fig3_1} either.
\end{example}
FSAs such as the one in Example \ref{eg3_1} are explicitly or implicitly adopted in several logic-based reward designs \zwcadd{such as reward machines~\cite{pmlr-v80-icarte18a}}. Such an FSA can be represented by the sketch in Fig.\ref{fig3_2} by introducing a local variable $\mathtt{q}$ as the FSA state. The key problem again is to determine appropriate values for the $\mathtt{?_{id}}$'s. We discuss the challenges facing this problem in the next section. 

%% file: arxiv/sec3.tex
\section{Problem Formulation}

In this section we augment the concept of sketching. Then we characterize the problem of Programmatic Reward Design (PRD) and discuss the challenges in PRD. 

\subsection{Sketching with Symbolic Constraints}
Given a sketch $\mathtt{e}$, a designer can put constraints on $\mathtt{\textbf{?}_e}$. A \textit{symbolic constraint} $\mathtt{c}$ is a combination of multiple predicates $\mathtt{c:=\mu\ |\ \neg \mu\ |\ c_1\wedge c_2\ |\ c_1\vee c_2}$ where the atomic predicate $\mathtt{\mu: \mathcal{H}\rightarrow \{\mathtt{\top}, \mathtt{\bot}\}}$ is the basic building block. Under Boolean operations, the predicates follow the semantics \eqref{sem3_2} where the expression $\mu[\textbf{h}/\textbf{?}_e]$ substitutes $\mathtt{\textbf{?}_e}$ in $\mathtt{\mu}$ with $\mathtt{\textbf{h}}$ to output Boolean values. A satisfying implementation of the mapping $\mathtt{eval: \{\mathtt{\top}, \mathtt{\bot}\}\rightarrow}\mathbb{R}$ is $\mathtt{eval}(\cdot):=2\mathcal{I}(\cdot) - 1$ where  $\mathcal{I}:\{\mathtt{\top}, \mathtt{\bot}\}\rightarrow\{0, 1\}$ is an indicator function. A symbolic constraint $\mathtt{c}$ is satisfied if $\mathtt{[\![c]\!](\textbf{h})\geq 0}$. Table \ref{tab1} instantiates two predicates $\mathtt{c_1, c_2}$ for Example \ref{eg2_1}. Then $\mathtt{c=c_1\wedge c_2}$ can be a symbolic constraint for Example \ref{eg2_1}. For simplicity, we only consider conjunctions of atomic predicates in this paper. \zwccross{The conjunction of all the predicates constitutes a symbolic constraint.}
Now we are now ready to state the Programmatic Reward Design problem as in Definition \ref{def1}. 
\begin{table}[htbp]
\centering
\begin{tabular}{l|l}
    \textbf{Properties} & \textbf{Predicates} \\
    $\mathtt{[c_1]}$Reward reaching the goal & $\mathtt{\mathtt{\wedge^{5}_{\mathtt{id}=1}(\mathtt{?_{id}} \leq \mathtt{?_1})}}$ \\
    $\mathtt{[c_2]}$Penalize dropping unused key & $\mathtt{\mathtt{?_5+?_4\leq  0}}$\\
\end{tabular}
\caption{The correspondence between two desired properties and the predicates for the sketch in Fig.\ref{fig1_2}. }
\label{tab1}
\end{table}
 \begin{eqnarray}
& \mathtt{[\![\mu]\!](\textbf{h}) := eval(\mu[\textbf{h}/\textbf{?}_e])\qquad [\![\neg \mu]\!](\textbf{h}):=-[\![\mu]\!](\textbf{h})}&\nonumber\\
&\mathtt{[\![c_1\wedge c_2]\!](\textbf{h}):= \min ([\![c_1]\!](\textbf{h}), [\![c_2]\!](\textbf{h}))}& \nonumber\\
&\mathtt{[\![c_1\vee c_2]\!](\textbf{h}):= \max ([\![c_1]\!](\textbf{h}), [\![c_2]\!](\textbf{h}))}& \label{sem3_2}
\end{eqnarray}

\begin{definition} [\textbf{Programmatic Reward Design (PRD)}]\label{def1}
For an RL task, a {PRD} problem is a tuple $\mathtt{\langle e, \textbf{?}_e, \mathcal{H}, c\rangle}$ where $\mathtt{e}$ is a \textit{sketch} with \textit{holes} $\mathtt{\textbf{?}_e}$ that takes values from set $\mathtt{\mathcal{H}}$; $\mathtt{c}$ is a symbolic constraint. The solution to a {PRD} problem $\mathtt{\langle e, \textbf{?}_e, \mathcal{H},  c\rangle}$ is any valid program $l\mathtt{\triangleq e[\textbf{h}/\textbf{?}_e]}$ subject to $\mathtt{\textbf{h}\in \mathcal{H}\wedge [\![c]\!](\textbf{h})\geq 0}$.
\end{definition}
\subsection{Challenges in PRD}

We note that solving the PRD problem does not 
\camadd{guarantee} that the resulting reward function will be effective.
The challenge of assigning proper values to the holes is faced not only by PRD but also by other symbolic reward design approaches. We use Example~\ref{eg3_1} to illustrate \camadd{two} aspects of this challenge. 
\textit{A) Goal-driven rewards.} The reward function specified via logically guided approaches~\cite{hasanbeig2019logically} can usually be translated into transition systems similar to the one in Fig.\ref{fig3_1}. Many of those approaches end up only assigning non-trivial values to $?_1$ and $?_4$ while equalizing the rewards for all the other transitions. However, when $p_1$ and $p_4$ are extremely small, e.g. $p_1,p_4 \ll p_2, p_3$, such goal-driven reward assignment barely provide any guidance when the agent explores the environment. \noindent\textit{B) Unknown dynamics.} The reward shaping approach from \cite{ijcai2019-840} adopts reward structures similar to the one in Fig.\ref{fig3_1} but allocates rewards to all transitions {while ignoring the dynamics of the environment.} This approach may result in \zwccross{other }inefficiency in large environments. For instance, if $p_2p_4< p_1\ll p_2$, a global optimal policy would only aim at triggering one transition $A\rightarrow B$. However, the shaped reward may assign a non-trivial reward to $?_{2}$, causing the RL algorithm to spend excessive training episodes on a local-optimal policy that lingers \zwcchange{around}{over} $A\rightarrow C\rightarrow B$. 

Given a sketch such as the one in Fig.\ref{fig3_2}, a PRD designer would also face the those challenges. Thus we assume that a PRD problem is accompanied by a set of demonstrations that show how an expert can accomplish the task, similar to the setting of IRL.
These demonstrations will help narrow down the solutions to the PRD problem. We hereby propose Programmatic Reward Design by Example (PRDBE).

%% file: arxiv/sec4.tex
\section{Programmatic Reward Design by Example}

Similar to Bayesian GAIL introduced in the Background section, we consider a probabilistic inference perspective for formulating PRDBE. \camadd{We first introduce a term to correlate programmatic reward functions with trajectory distributions.} 
\begin{definition}[\textbf{Nominal Trajectory Distribution}]\label{def4_1}
Given a programmatic reward function $l$, a \textbf{nominal trajectory distribution} of $l$ is $\hat{p}(\tau|l)=p(\tau)\exp(l(\tau))$ where $p(\tau)$ is the probability of sampling $\tau$ under the passive dynamics of the environment; $l(\tau)$ is short for $\sum_t [\![l]\!](\tau)[t]$. Furthermore, a \textbf{normalized nominal trajectory distribution} of $l$ is $p(\tau|l)=p(\tau)\exp(l(\tau))/Z_l$ where $Z_l=\sum_\tau p(\tau)\hat{p}(\tau|l)$. 
\end{definition}
The nominal trajectory distribution $\hat{p}(\tau|l)$ can be viewed as a proxy of a possibly non-Markov policy $\pi_l(a^{(t)}|s^{(t)} )\propto \exp([\![l]\!](\tau)[t])$. Such a policy trivially minimizes $J_{RL}(\pi_l)$\cite{levine2018reinforcement}, the RL loss described in the Background section. Intuitively, we search for an $l^*$ such that $\pi_{l^*}$ \textit{matches} $\pi_E$. Given this intuition, we formally define the problem of Programmatic Reward Design by Example in Definition \ref{def4_2}. 
\begin{definition} [\textbf{Programmatic Reward Design by Example (PRDBE)}]\label{def4_2}
Given a set of expert demonstrated trajectories $E$
and a PRD problem $\mathtt{\langle e, \textbf{?}_e, \mathcal{H}, c\rangle}$,
the PRDBE problem $\mathtt{\langle  e, \textbf{?}_e, \mathcal{H}, c}, E\rangle$ is to find a solution $ l^*$ to the PRD problem such that for any $\tau$ the nominal trajectory distribution satisfies $\hat{p}(\tau|l^*)=p(\tau|l^*)= p_E(\tau)$ where $p_E$ is the probability of sampling $\tau$ from $E$.\li{why use $\equiv$?}
\end{definition}

 However, solving the PRDBE problem requires addressing the following challenges: \textit{a) the set of solutions to the PRD problem may not contain a satisfying solution $l^*$ for PRDBE, and b) the sketch may not be differentiable w.r.t the holes}. In other words, there may not exist a perfect solution to the PRDBE problem and  \zwcchange{existing neural network training routine}{gradient-based optimizations}\li{do you just mean gradient-based optimization?} may not be readily applicable to PRDBE. To overcome these issues, we propose a learning framework with a relaxed objective.

\subsection{A Generative Adversarial Learning Framework}
\li{italicize the most important ideas and results}
Our learning framework realizes the probability matching between $\hat{p}(\tau|l), p(\tau|l)$ and $p_E(\tau)$ in a \textit{generative-adversarial} fashion. 
It searches for an $l$ such that even the best discriminator represented with a reward function $f$ as mentioned in the Background section cannot distinguish trajectories sampled by $p_E$ from those by $p(\tau|l)$. 
\camadd{Given an intermediate, learned reward function $f$, }
\camadd{while Bayesian GAIL trains a $\pi_A$ to minimize the log-likelihood \eqref{eq1_0} of correct classifications between agent trajectory $\tau_A$ and expert trajectory $\tau_E$}, we learn an $l$ to maximize the log-likelihood of false classifications, i.e., $\log p(1_A, 0_E|l, f; E)=\log\sum_{\tau_A}  p(\tau_A|l)p(1_A|\tau_A; l, f)\sum_{\tau_E}p_E(\tau_E)p(0_E|\tau_E; l, f)$ with $p(1_A|\tau_A;l, f):=\prod^T_{t=0}\frac{exp(f(s^{(t)}_A,a^{(t)}_A))}{exp(f(s^{(t)}_A,a^{(t)}_A)) + \exp([\![l]\!](\tau_A)[t])}$ and $p(0_E|\tau_E;l, f):=\prod^T_{t=0}\frac{exp([\![l]\!](\tau_E)[t])}{exp(f(s^{(t)}_E,a^{(t)}_E)) + \exp([\![l]\!](\tau_E)[t])}$ being the \textit{discriminators}; $p(\tau_A|l)$ being the \textit{generator} of $\tau_A$'s. 
Since the $l$'s are non-differentiable, we do not directly optimize $l$. 
Recalling that $\mathcal{L}$ is the program space induced by $\mathcal{H}$, we optimize a sampler $q:\mathcal{L}\rightarrow [0, 1]$ to concentrate the distribution density on those candidate programs $l$'s that incur high values of $\log p(1_A, 0_E|l, f; E)$. 
Due to the introduction of $q$, the log-likelihood of false classification becomes $\log p(1_A, 0_E|q, f; E)\geq  \sum_{l\in\mathcal{L}}q(l)\log p(1_A, 0_E|l, f; E)$ which is further lower-bounded by \eqref{eq4_7}. We then introduce an agent policy $\pi_A$ for importance sampling of $p(\tau_A|l)$ in \eqref{eq4_8}. Thus $q$ and $\pi_A$ together constitute the \textit{generator}. 
After cancelling out the passive-dynamics induced $p(\tau_A)$ in the nominator and denominator as in \eqref{eq4_9}, \li{expose?}\zwcchange{the normalization constant $Z_l$ in \eqref{eq4_9} and then elide}{we handle the normalization constant $Z_l$}\li{elide?} via self-normalized importance sampling~\cite{mcbook} as in \eqref{eq4_10} by i.i.d. sampling a set $\{\tau_{A,i}\}^m_{i=1}$ of trajectories with $\pi_A$. We refer to \eqref{eq4_10} as the generative objective $J_{gen}(q)$.

\begin{eqnarray}
&&\sum_{l\in\mathcal{L}}q(l) \log p(1_A, 0_E|\pi_{l};l, f; E)\nonumber\\
&\geq&{\underset{l\sim q}{\mathbb{E}}}\Big\{\underset{\tau_A\sim p(\tau_A|l)}{ \mathbb{E}}\big[\log p(1_A|\tau_A; l,f)\big] + \nonumber\\
&& \qquad\qquad\qquad\quad \underset{\tau_E\sim p_E}{ \mathbb{E}}\big[ \log p(0_E|\tau_E; l, f)\big]\Big\}\label{eq4_7}\\
&= &{\underset{l\sim q}{\mathbb{E}}}\Big\{\underset{\tau_A\sim \pi_{A}}{ \mathbb{E}}\big[\frac{p(\tau_A|l)}{p(\tau_A|\pi_A)}\log p(1_A|\tau_A; l, f) \big]+ \nonumber\\
&& \qquad\qquad\qquad\quad \underset{\tau_E\sim p_E}{ \mathbb{E}}\big[\log p(0_E|\tau_E; l, f)\big]\Big\}\label{eq4_8}\\
&\geq & {\underset{l\sim q}{\mathbb{E}}} \Big\{\underset{\tau_A\sim \pi_A}{ \mathbb{E}}\Big[\frac{\exp(l(\tau_A))}{Z_l\pi_A(\tau_A)}\log p(1_A|\tau_A; l, f) \Big]+ \nonumber\\
&&\qquad\qquad\qquad\quad \underset{\tau_E\sim p_E}{ \mathbb{E}}\big[\log p(0_E|\tau_E; l, f)\big]\Big\}\label{eq4_9}\\
&\approx& {\underset{l\sim q}{\mathbb{E}}}\Big\{\frac{\sum\limits^m_{i=1}\big(\frac{\exp(l(\tau_{A,i}))}{\pi_A(\tau_{A,i})}\big) \log p(1_A|\tau_{A,i}; l, f)}{\sum\limits^m_{i=1}\frac{\exp(l(\tau_{A,i}))}{\pi_A(\tau_{A,i})}}+\nonumber\\
&& \qquad\qquad\qquad\quad \underset{{\tau_E\sim p_E}}{ \mathbb{E}}\big[\log p(0_E|\tau_E; l, f)\big]\Big\}\label{eq4_10}\\
&\triangleq &J_{gen}(q)\nonumber
\end{eqnarray}

Next, we note that the existence of symbolic constraint $\mathtt{c}$ imposes a prior $p(l|\mathtt{c})$ over the search space of $l$. We let $p(l|\mathtt{c})$ be a uniform distribution among the programs that do not violate $\mathtt{c}$, i.e., $\{l\triangleq\mathtt{e[\textbf{h}/\textbf{?}_e]}|\mathtt{[\![c]\!](\textbf{h})\geq 0}\}\subseteq \mathcal{L}$ while being {zero anywhere else}. Then the objective of our learning framework for $q$ becomes minimizing $D_{KL}\big(q(l)|| p(l| 1_A, 0_E, f; E, \mathtt{c})\big)$ where $p(l| 1_A, 0_E, f; E, \mathtt{c})=\frac{p(1_A, 0_E|l, f; E)p(l|\mathtt{c})}{p(1_A, 0_E|f;E)}$. We minimize this KL-divergence by maximizing its evidence lower-bound (ELBO) as in \eqref{eq4_11} which equals the sum of $J_{gen}(q)$, an entropy term $H(q)$ and a $J_{\mathtt{c}}(q)\triangleq\underset{l\sim q}{\mathbb{E}}[\log p(l|\mathtt{c})]$. 
\begin{eqnarray}
ELBO(q)&=& J_{gen}(q) - D_{KL}\big(q(l)||p(l|\mathtt{c})\big)\nonumber\\
&=& H(q) + J_{c}(q) + J_{gen}(q)\label{eq4_11}
\end{eqnarray}

In our implementation, \camadd{the holes in the designed sketches are all unknown reals, i.e., $\mathcal{H}=\mathbb{R}^{|\textbf{?}_e|}$.} Rather than sampling from $\mathcal{L}$, we can use a neural network $q_\varphi$ to specify the mean and variance of a $|\mathtt{\textbf{?}_e}|$-dimensional multivariate Gaussian distribution in $\mathcal{H}$. 
We use $q_\varphi(l)$ to denote the probability of this Gaussian distribution producing an $\mathtt{\textbf{h}\in \mathcal{H}}$ such that $l=\mathtt{ e[\textbf{h}/\textbf{?}_e]}$. The mean of the Gaussian corresponds to the most likely $l^*$ which we use to train a neural network policy $\pi_\phi$ as the agent policy $\pi_A$ in \eqref{eq4_10}. To calculate the gradient of $J_{gen}$, we use the logarithmic trick~\cite{peters2008reinforcement}, i.e., $\nabla_{\varphi_n}{\mathbb{E}}_{l\sim q_{\varphi_n}}[\cdot]\approx \frac{1}{K}\sum^K_{k=1} \nabla_{\varphi_n}\log q_{\varphi_n}(l_k)[\cdot]$. \camadd{We note that \camadd{$J_c(q)$ is $-\infty$ once the support of $q$ contains an $l$ that violates $\mathtt{c}$}. Since the support of the Gaussian specified by $q_\varphi$ is the real space, $J_c(q_\varphi)$ can always be $-\infty$. 
Hence, we relax $J_c$ by using a ReLU function and a cross-entropy loss to only penalize $q_\varphi$ if $l^*$ violates $\mathtt{c}$.}
We also train a neural reward function $f_\theta$ as the $f$ in \eqref{eq4_10} by maximize an adversarial objective $J_{adv}(f_\theta)$ that differs from the generative objective $J_{gen}$ by changing $1_A, 0_E$ into $0_A, 1_E$ {such that} $p(0_A|\tau_A;l, f_\theta):=\prod^T_{t=0}\frac{exp([\![l]\!](\tau_A)[t])}{exp(f_\theta(s^{(t)}_A,a^{(t)}_A)) + \exp([\![l]\!](\tau_A)[t])}$ and $p(1_E|\tau_E;l, f_\theta):=\prod^T_{t=0}\frac{f_\theta(s^{(t)}_E, a^{(t)}_E)}{exp(f_\theta(s^{(t)}_E,a^{(t)}_E)) + \exp([\![l]\!](\tau_E)[t])}$. The algorithm is summarized in Algorithm 1 and depicted in Fig.~\ref{fig4_2}. In a nutshell, this algorithm iteratively samples trajectories with $\pi_\phi$, train $\pi_\phi$ with $l^*$, then update $f_\theta$ and $q_\varphi$ respectively.


\begin{algorithm}[tb]
\caption{Generative Adversarial PRDBE}
\label{alg:algorithm}
\textbf{Input}: PRDBE tuple $\mathtt{\langle e,\textbf{?}_e, \mathcal{H}, c},E\rangle$, iteration number $N$, sample number $m, K$, optimization parameters $\alpha, \beta$ \\
\textbf{Output}: $l^*, \pi^*$
\begin{algorithmic}[1] 
\STATE \textbf{initialization}: program space $\mathtt{\mathcal{L}=\{e[\textbf{h}/\textbf{?}_e]|\textbf{h}\in\mathcal{H}\}}$; an agent policy $\pi_{\phi_0}$; neural reward function $f_{\theta_0}$; sampler $q_{\varphi_0}:\mathcal{L}\rightarrow [0, 1]$
\WHILE {iteration number $n=0, 1, \ldots, N$}
\STATE Sample trajectory set $\{\tau_{A,i}\}^m_{i=1}$ by using policy $\pi_{\phi_n}$
\STATE Calculate rewards $\{[\![l^*_n]\!](\tau_{A,i})\}^m_{i=1}$ with the most likely program $l^*_n=\arg\underset{l}{\max}\ q_{\varphi_n}(l)$
\STATE Update $\phi_n$ to $\phi_{n+1}$ using policy learning algorithm, e.g. PPO~\cite{DBLP:journals/corr/SchulmanWDRK17}
\STATE Sample $K$ programs $\{l_k\}^K_{i=1}$ by using $q_{\varphi_n}$
\STATE Calculate rewards $\{\{[\![l_k]\!](\tau_{A,i})\}^m_{i=1}\}^K_{k=1}$
\STATE Update $\theta_{n+1}\leftarrow \theta_{n} + \alpha \nabla_{\theta} J_{adv}(f_{\theta_n})$
\STATE Update $\varphi_{n+1}\leftarrow \varphi_n + \beta \nabla_{\varphi} ELBO(q_{\varphi_n})$ 
\ENDWHILE
\STATE \textbf{return} $l^*:=\arg\underset{l}{\max}\ q_{\varphi_N}(l)$ and $\pi^*:=\pi_{\phi_N}$
\end{algorithmic}
\end{algorithm}
\begin{figure}
     \centering
         \centering
         \includegraphics[height=3cm, width=7cm]{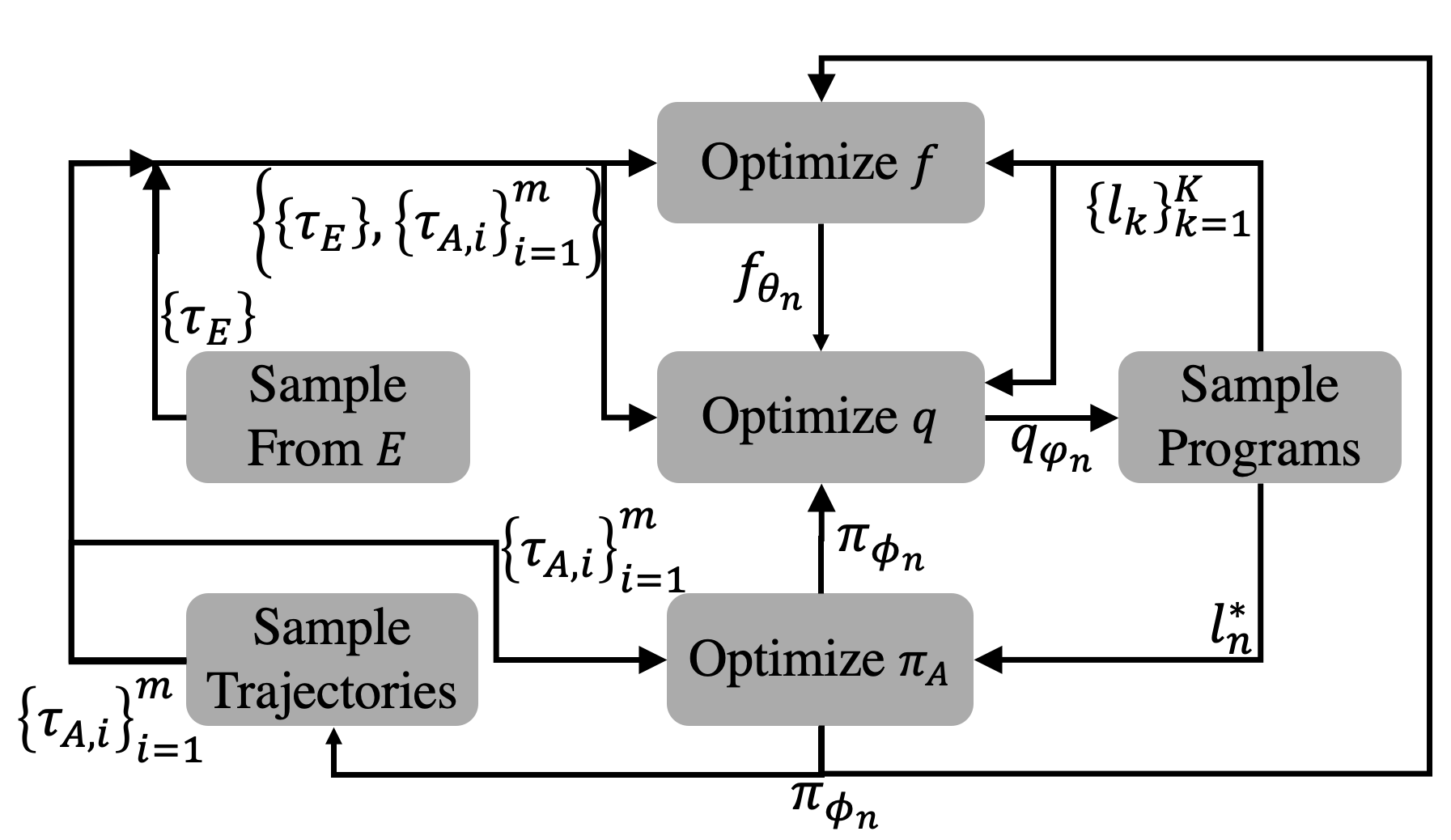}
         \caption{Flowchart for our learning framework}
         \label{fig4_2}
\end{figure}

%% file: arxiv/sec5.tex
\section{Experiments}
\begin{figure*}[t]
     \centering
     \begin{subfigure}[b]{0.24\textwidth}
         \centering
         \includegraphics[height=3.cm,width=3.2cm]{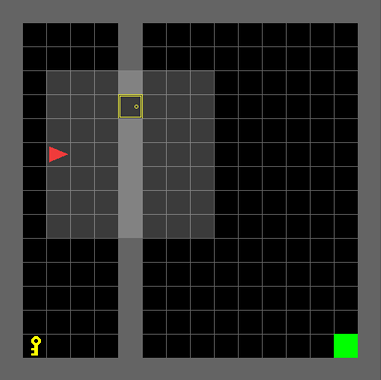}
         \caption{DoorKey-16x16}
         \label{fig5_1}
     \end{subfigure}
     \begin{subfigure}[b]{0.24\textwidth}
         \centering
         \includegraphics[height=3.1cm, width=4.5cm]{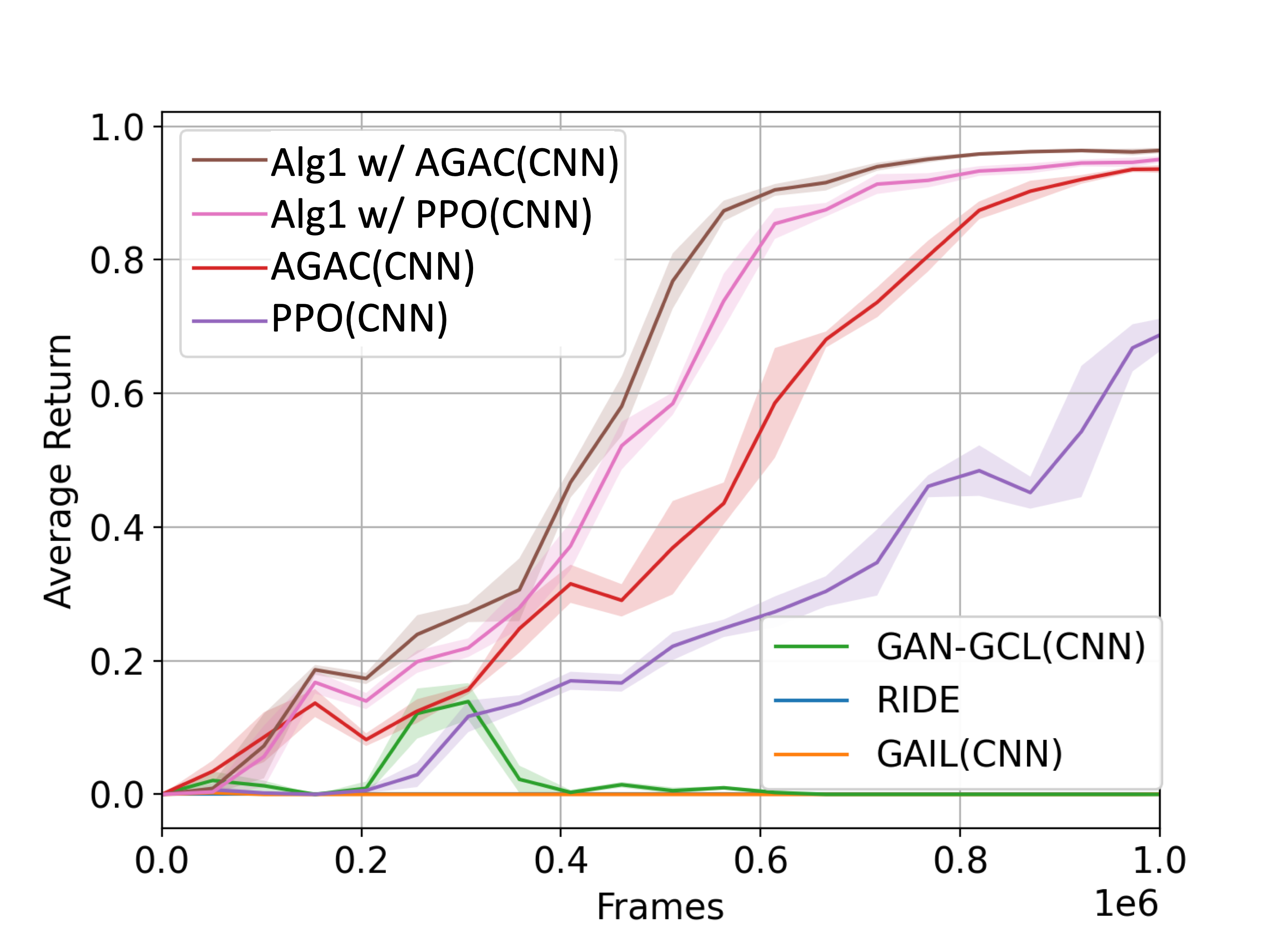}
         \caption{DoorKey-8x8-v0}
         \label{fig5_7}
     \end{subfigure}
     \begin{subfigure}[b]{0.24\textwidth}
         \centering
         \includegraphics[height=3.1cm, width=4.5cm]{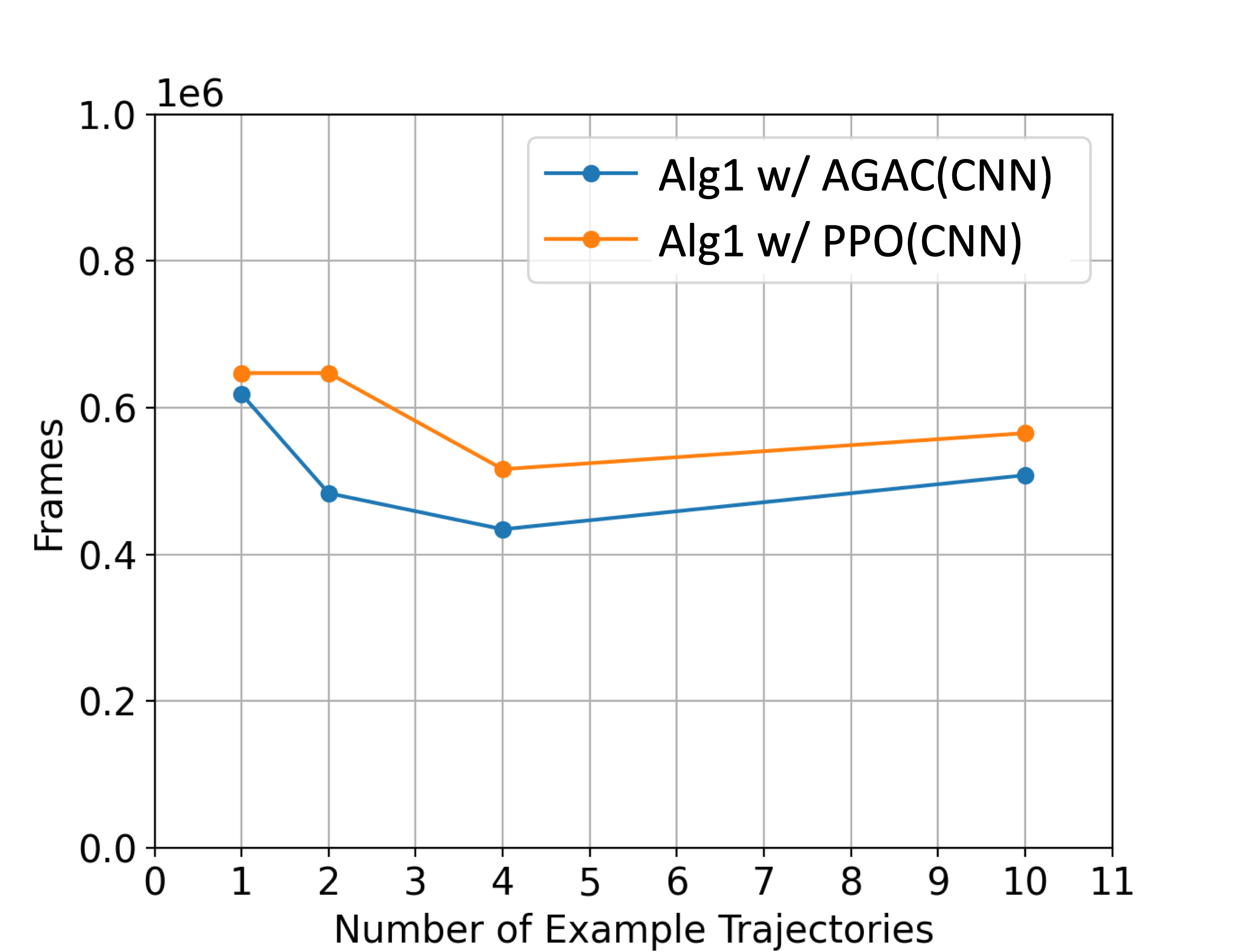}
         \caption{DoorKey-8x8-v0}
         \label{fig5_8}
     \end{subfigure}
      \begin{subfigure}[b]{0.24\textwidth}
         \centering
         \includegraphics[height=3.1cm, width=4.5cm]{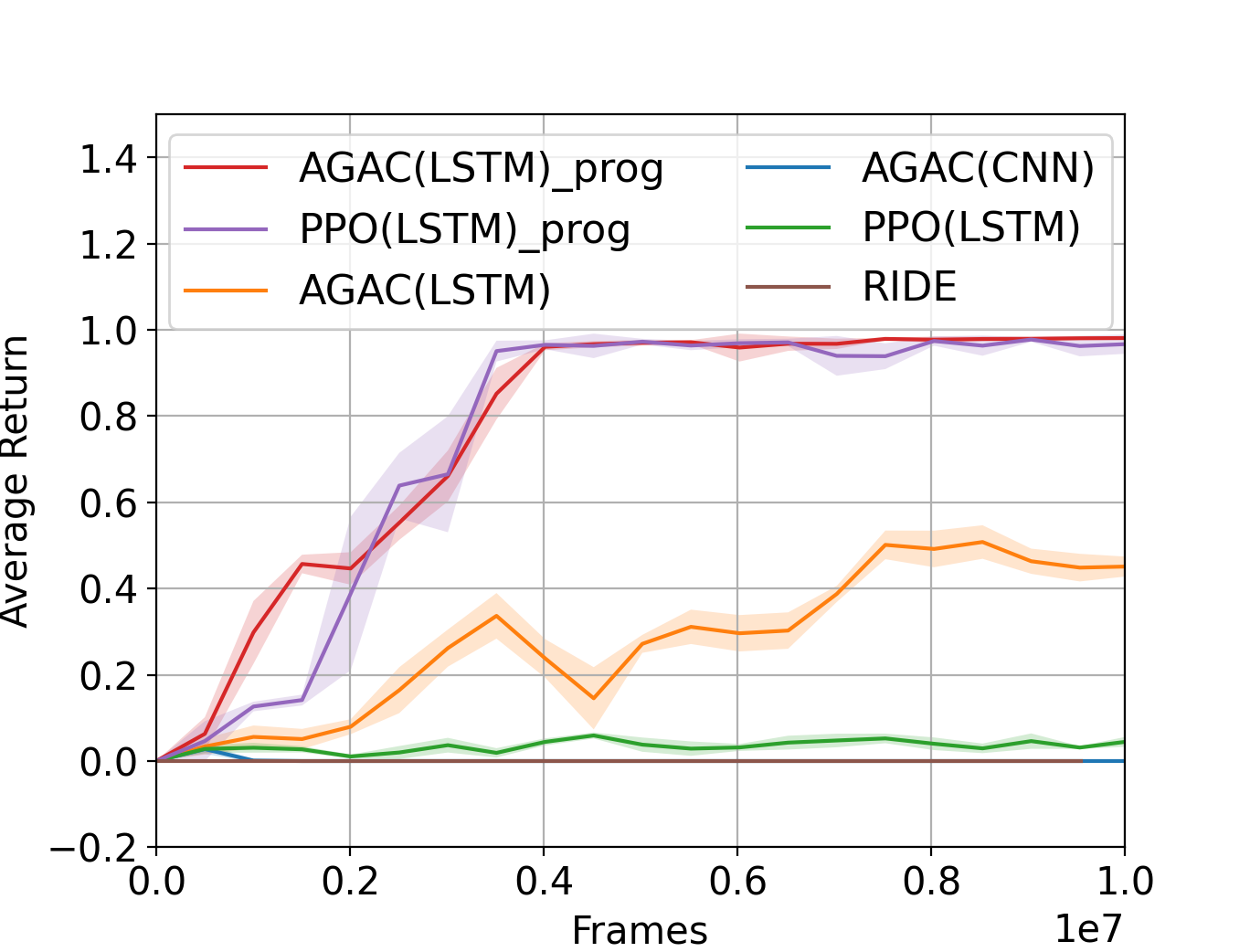}
         \caption{DoorKey-16x16-v0}
         \label{fig5_9}
     \end{subfigure}
      \begin{subfigure}[b]{0.24\textwidth}
         \centering
         \includegraphics[height=3.cm,width=3.2cm]{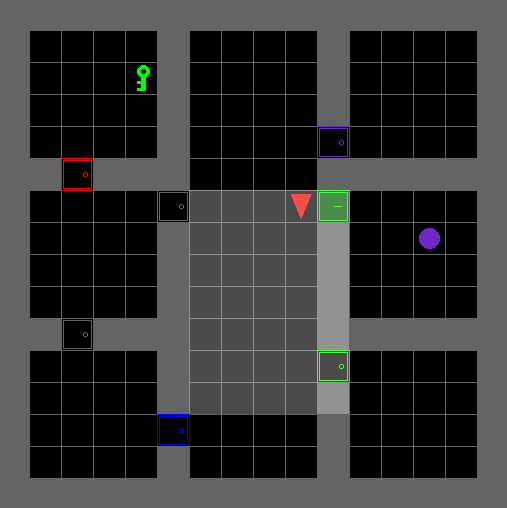}
         \caption{KeyCorridorS6R3}
         \label{fig5_1}
     \end{subfigure}
     \begin{subfigure}[b]{0.24\textwidth}
         \centering
         \includegraphics[height=3.1cm, width=4.5cm]{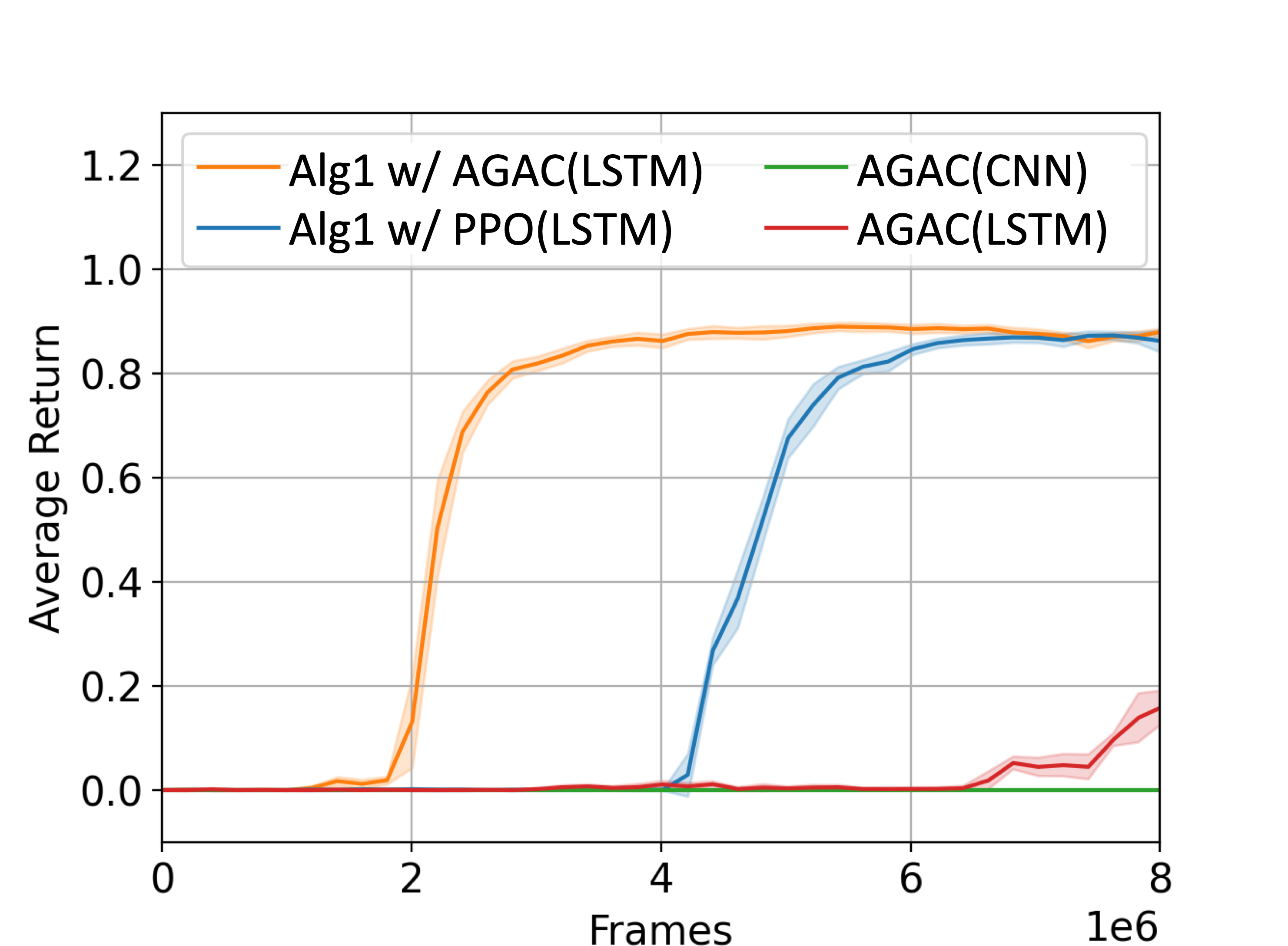}
         \caption{KeyCorridorS3R3}
         \label{fig5_3}
     \end{subfigure} 
     \begin{subfigure}[b]{0.24\textwidth}
         \centering
         \includegraphics[height=3.1cm, width=4.5cm]{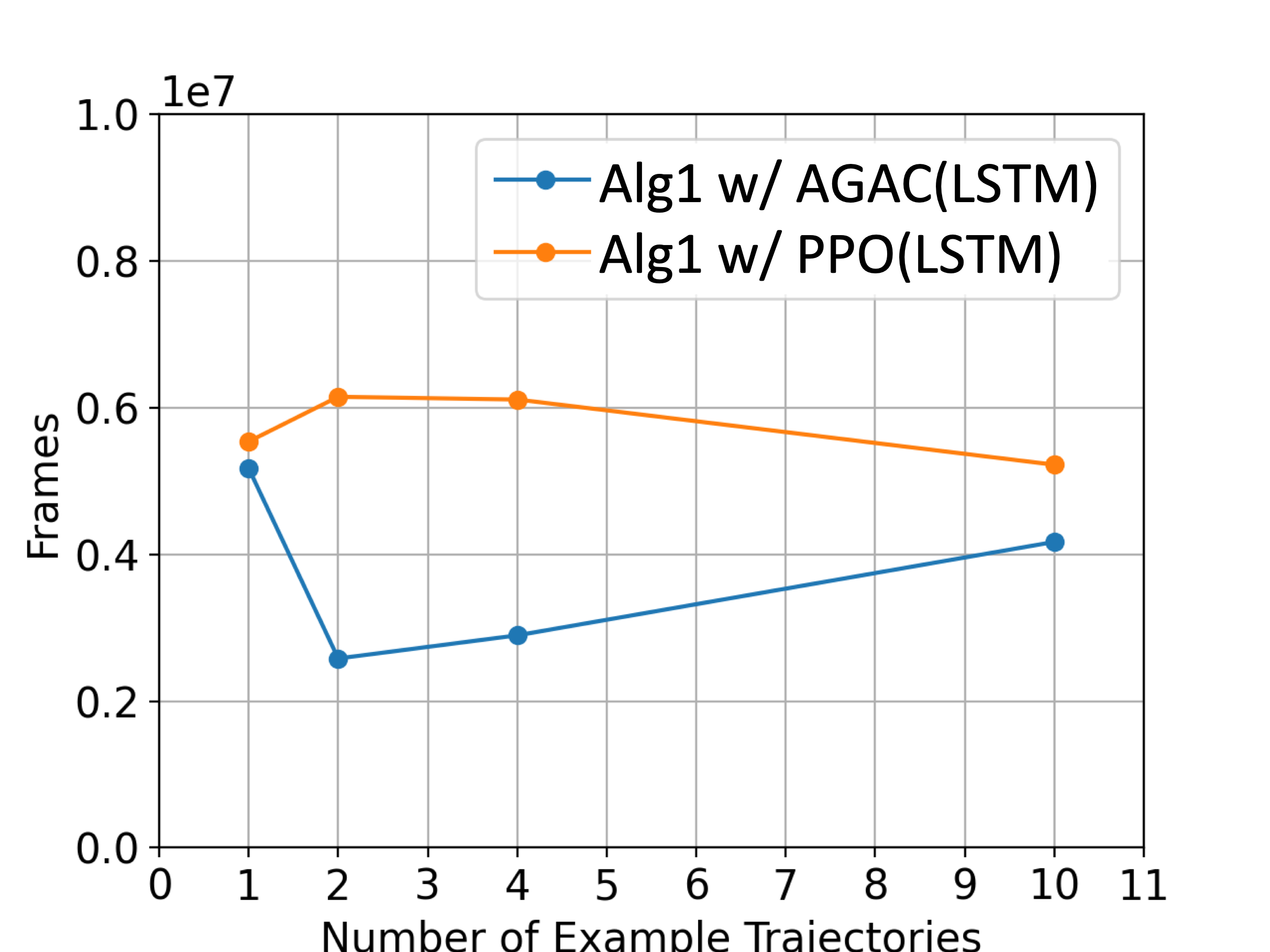}
         \caption{KeyCorridorS3R3}
         \label{fig5_4}
     \end{subfigure}
      \begin{subfigure}[b]{0.24\textwidth}
         \centering
         \includegraphics[height=3.1cm, width=4.5cm]{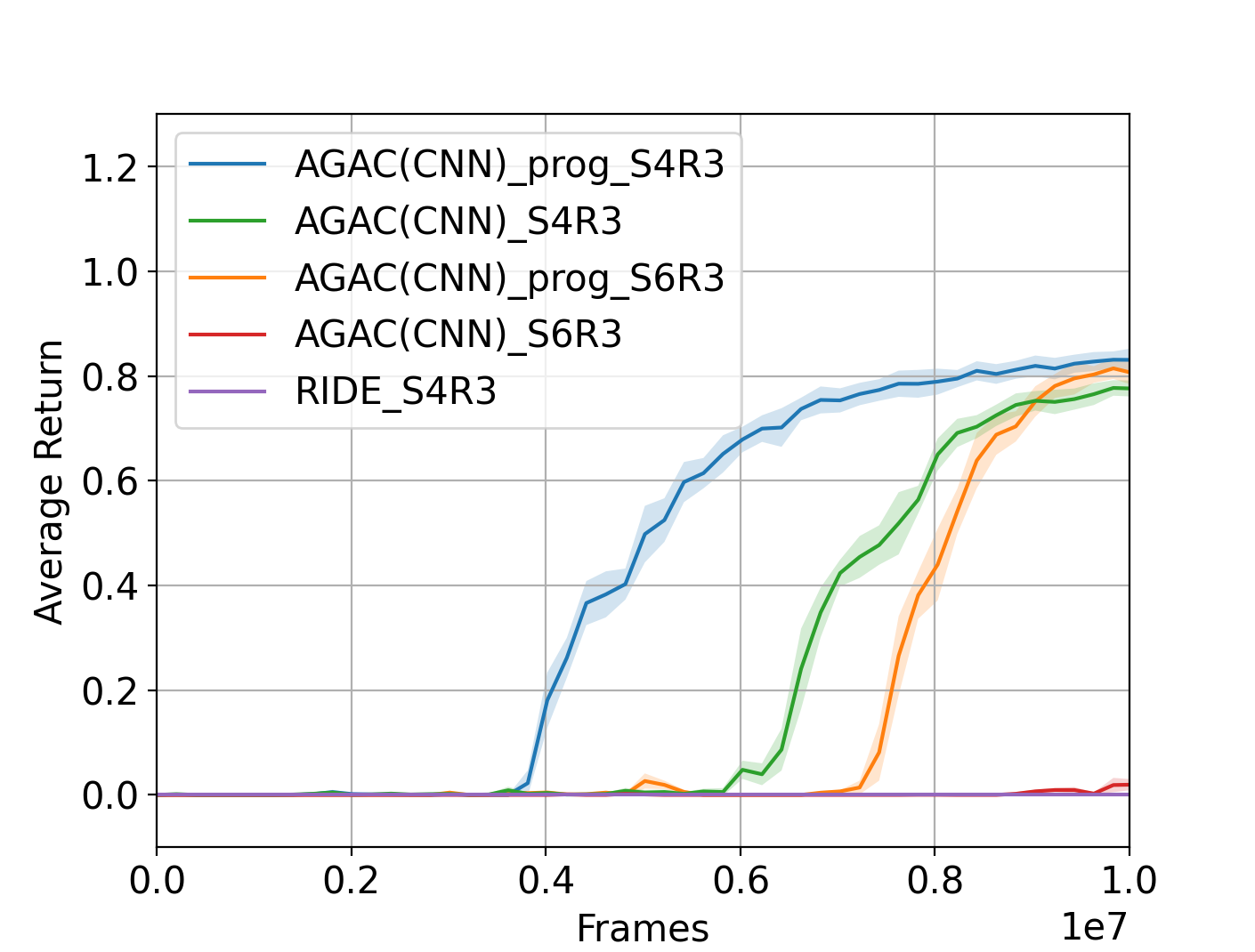}
         \caption{KeyCorridorS4/6R3}
         \label{fig5_5}
     \end{subfigure} 
     \begin{subfigure}[b]{0.24\textwidth}
         \centering
         \includegraphics[height=3.cm,width=3.2cm]{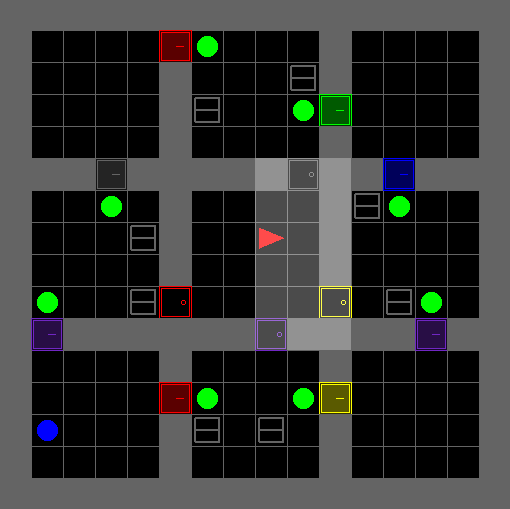}
         \caption{ObstructedMaze-Full}
         \label{fig5_2}
     \end{subfigure}
     \begin{subfigure}[b]{0.24\textwidth}
         \centering
         \includegraphics[height=3.1cm, width=4.5cm]{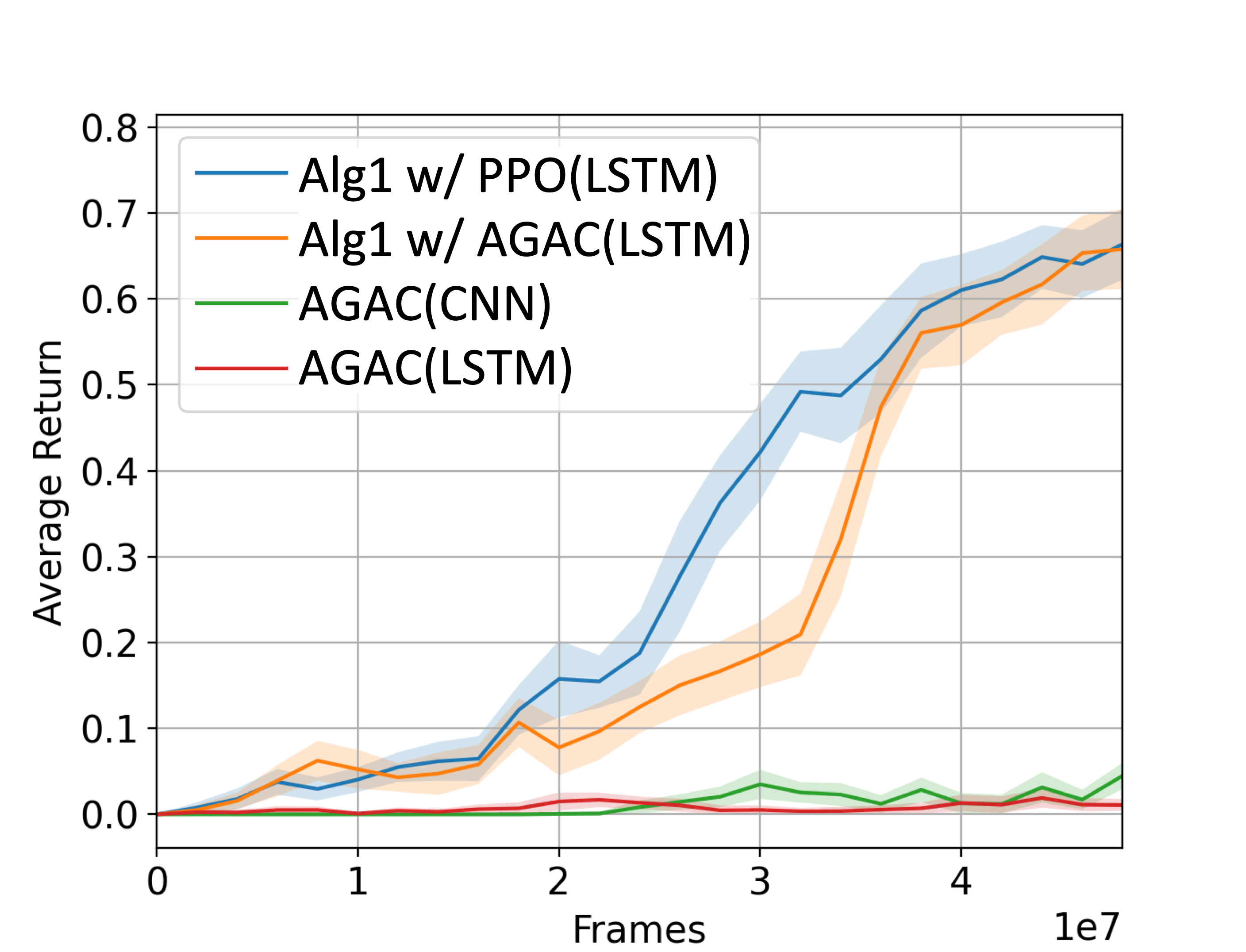}
         \caption{ObstructedMaze-2Dlhb}
         \label{fig5_6}
     \end{subfigure}
      \begin{subfigure}[b]{0.24\textwidth}
         \centering
         \includegraphics[height=3.1cm, width=4.5cm]{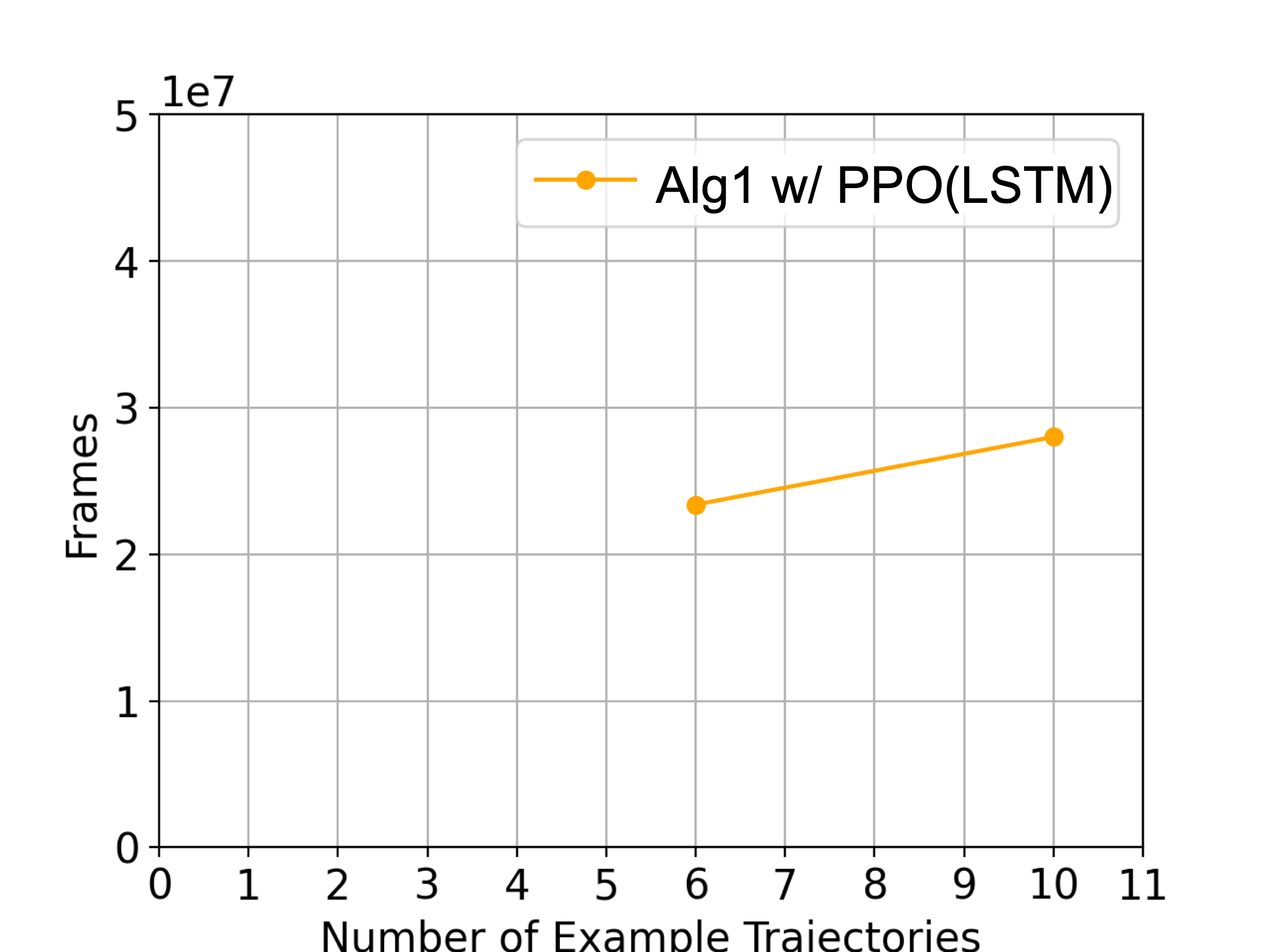}
         \caption{ObstructedMaze-2Dlhb}
         \label{fig5_11}
     \end{subfigure} 
     \begin{subfigure}[b]{0.24\textwidth}
         \centering
         \includegraphics[height=3.1cm, width=4.5cm]{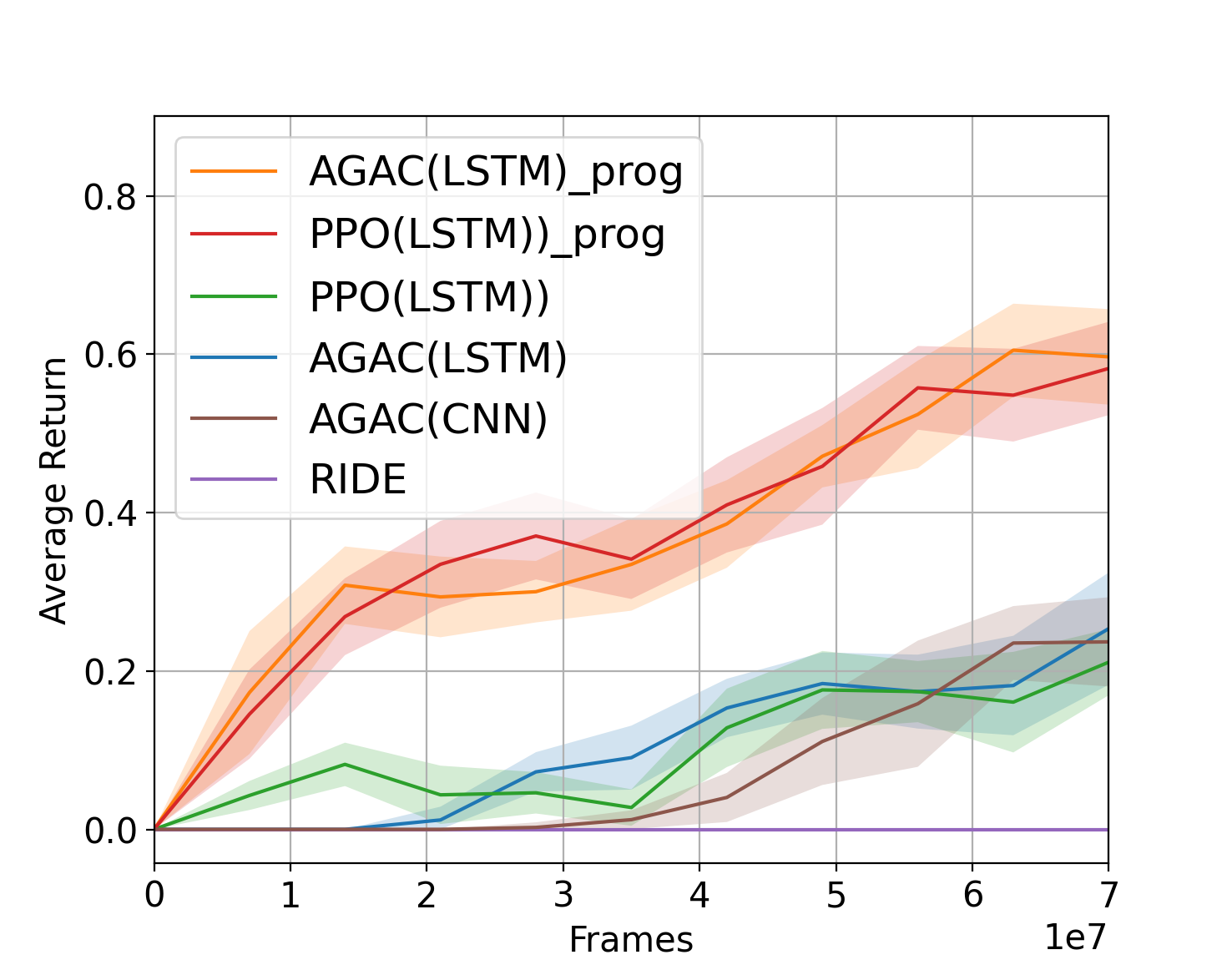}
         \caption{ObstructedMaze-Full}
         \label{fig5_10}
     \end{subfigure} 
     \caption{Frames are number of interactions with the environment. Average return is the average default reward achieved over a series of episodes and is no larger than $1$. Alg1 w/ AGAC/PPO indicates using AGAC or PPO as the policy learning algorithm in line 5 of Algorithm 1. AGAC/PPO\textit{\_prog} indicates training an AGAC or PPO agent with reward provided by a programmatic reward function \textit{prog}. CNN and LSTM indicate the structures of the actor-critic netowrks.} 
\end{figure*}
In this section, we experimentally investigate: \textbf{A. Performance:} whether {Algorithm 1} can efficiently train an agent policy to attain high performance; \textbf{B. Example Efficiency:} whether Algorithm 1 can work with only a few demonstrated trajectories;  \textbf{C. Generalization:} whether the programs learned via Algorithm 1 in one environment can generalize across different environments of the same task. 

\noindent\textbf{Benchmark.} We select from the MiniGrid environments~\cite{gym_minigrid} three challenging RL tasks in ascending order of difficulty: the basic setup and the default reward function for the \textit{Door-Key} task have been described in Example \ref{eg2_1}; \textit{KeyCorridor} shown in Fig.\ref{fig5_1} and \textit{ObstructedMaze} in Fig.\ref{fig5_2} both require the agent to travel from room to room to find the key(s) to pick up a colored ball in some locked room. In KeyCorridor, all but one room are unlocked. In ObstructedMaze, most of the rooms are locked and obstructed by green balls\zwcchange{; furthermore,}{ and} the keys are hidden in grey boxes which the agent must open first. Note that the agent can only carry one object at a time, which makes picking up and dropping the same objects multiple times almost inevitable. 
\zwccross{The agent's 7x7 field of vision makes the state space overwhelmingly larger than it appears.}  
The environments can vary in size by changing the number of rooms and tiles (e.g. DoorKey-8x8 vs. DoorKey-16x16). 
The placements of the objects and doors are randomized in each instance of an environment (e.g. ObstructedMaze-Full). 
\zwccross{as those in Fig.\ref{fig1_2}.} 
By default, the environment only returns a reward when the agent reaches the goal tile or picks up the targeted ball, making exploration in the environment particularly challenging. Besides, since only the $7\times7$ tiles in front of the agent are observable, memorization also turns out to be challenging. Those environments have been extensively used to benchmark exploration-driven, curiosity-driven and intrinsic reward driven RL agents. In this paper, we use programmatic reward function to train RL agents to accomplish those tasks.

\noindent\textbf{Baselines.} \camadd{We compare Algorithm 1 with IRL algorithms, GAN-GCL~\cite{fu2017learning} and GAIL~\cite{ho2016generative} to answer question \textbf{A}. We use PPO~\cite{DBLP:journals/corr/SchulmanWDRK17}, and AGAC~\cite{flet-berliac2021adversarially} for RL training in line 5 of Algorithm 1. 
We answer question \textbf{B} by varying the number of demonstrated trajectories when running Algorithm 1. We answer question \textbf{C} by using the programmatic reward functions learned via Algorithm 1 in small environments to train RL agents in larger environments (the results are annotated with AGAC/PPO\textit{\_prog}). In all three tasks, we provide the results of running the aforementioned RL algorithms as well as an intrinsic-reward augmented RL algorithm, RIDE~\cite{Raileanu2020RIDE:}, with the default rewards.} The sketches and symbolic constraints are in a similar form to those described in Example \ref{eg2_1}. We implement an LSTM reward function $f_\theta$ for Algorithm 1 and IRL baselines; an MLP $q_\varphi$; a CNN version and an LSTM version of the actor-critic RL agent $\pi_\phi$ respectively but only report the version with the higher performance. 

\noindent \textbf{DoorKey.} We use $10$ example trajectories demonstrated in a DoorKey-8x8 environment.
The PRDBE problem and its solution are both annotated as \textit{prog}. 
In Fig.\ref{fig5_7}, running Algorithm 1 by using PPO and AGAC in line 5 respectively produces a high-performance policy with fewer frames than by training PPO or AGAC with the default reward. 
On the other hand, RIDE, GAN-GCL and GAIL all fail with close-to-zero returns.
In Fig.\ref{fig5_8} we reduce the number of examples from $10$ to $1$ and it does not affect
the number of frames that Algorithm 1 needs to produce a policy with average return of at least $0.8$, regardless of whether PPO or AGAC is used in line 5. \camadd{ This shows that Algorithm 1 is \textit{example efficient}.}
In Fig.\ref{fig5_5}, we use the learned program to train PPO and AGAC agents in a larger DoorKey environment and achieve high performances with fewer frames than training PPO, AGAC or RIDE with the default reward.\camadd{This shows that the learned programmatic reward \textit{generalizes} well across environments.}

\noindent \textbf{KeyCorridor.} We use $10$ example trajectories demonstrated in a $6\times6$ KeyCorridorS3R3 environment. 
In Fig.\ref{fig5_3}, by using PPO and AGAC in line 5 of Algorithm 1, we respectively obtain high performance with significantly fewer frames than by training AGAC with the default reward.  
We omit GAIL and GAN-GCL because both fail in this task. As shown in Fig.\ref{fig5_4}, reducing the number of examples (to $1$) does not affect the performance of Algorithm 1. In Fig.\ref{fig5_5}, we use the learned program to train AGAC agents in two larger environments, $10\times 10$ KeyCorridorS4R3 and $16\times16$ KeyCorridorS6R3, and achieve high performances with fewer frames than training AGAC with the default reward. Note that when using the default reward, AGAC has the prior SOTA performance for this task.

\noindent \textbf{ObstructedMaze.} We use $10$ example trajectories demonstrated in a two-room ObstructedMaze-2Dhlb environment. 
When training the PPO agent, we discount PPO by episodic state visitation counts as in many exploration-driven approaches including AGAC. 
In Fig.\ref{fig5_6} we show that Algorithm 1 can produce high-performance policies with fewer frames than AGAC trained with the default reward. Since AGAC(CNN) is the SOTA for this task, we do not show the results of other methods. In Fig.\ref{fig5_5}, we use the learned program to train AGAC agents in a larger environment with as many as 9 rooms, ObstructedMaze-Full, and achieve higher performances with fewer frames than other methods trained with the default reward. Regarding the number of demonstrations, as shown in Fig.\ref{fig5_11}, the performance of Algorithm 1 does not change much when the number is decreased to $6$, but drops drastically when the number is further decreased possibly due to the higher complexity of this task. 

\begin{figure*}
     \centering
     \begin{subfigure}[b]{0.32\textwidth}
         \centering
        \includegraphics[height=4.1cm, width=5.5cm]{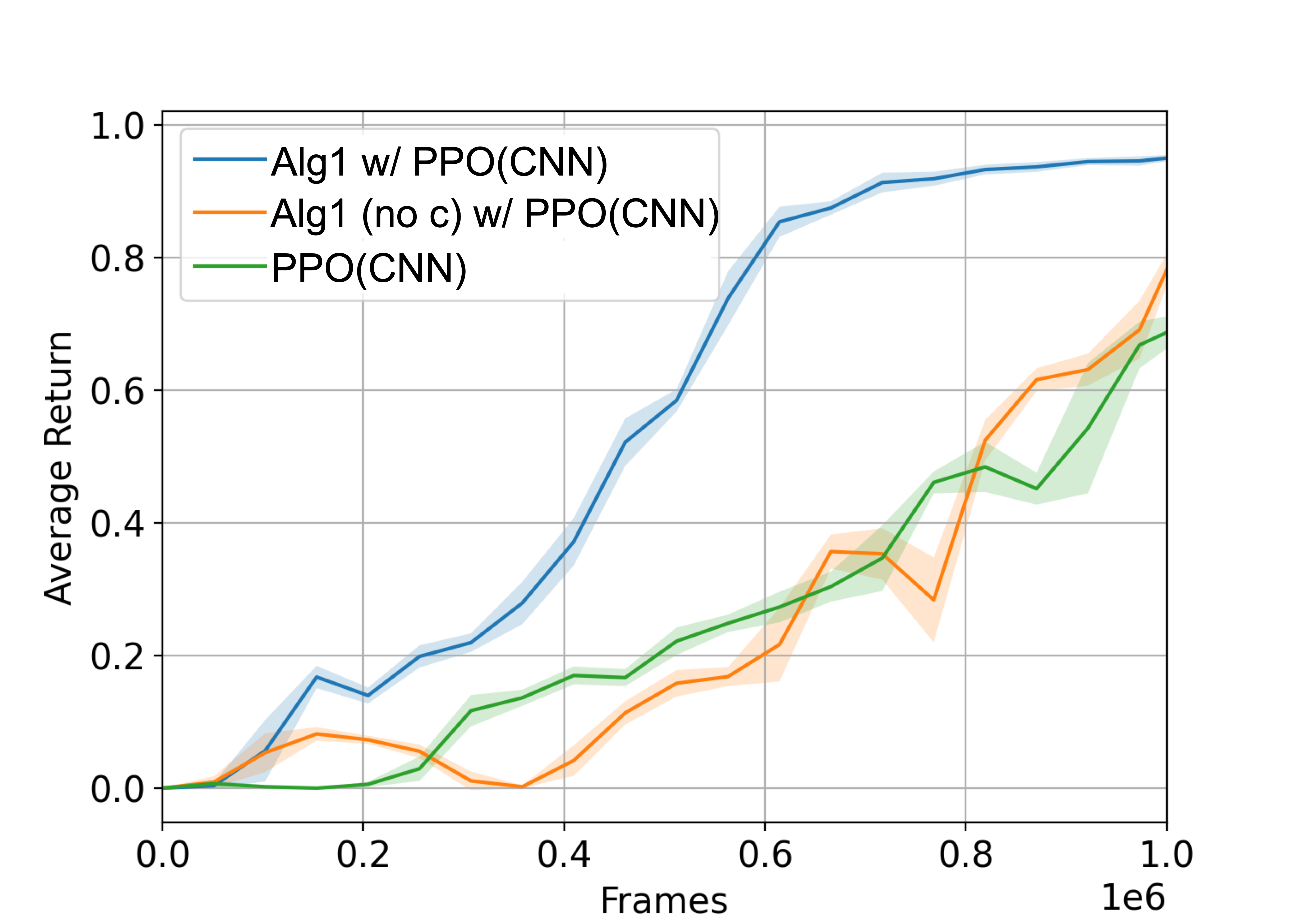}
         \caption{DoorKey-8x8}
         \label{fig6_3}
     \end{subfigure}  
     \begin{subfigure}[b]{0.32\textwidth}
         \centering
        \includegraphics[height=4.1cm, width=5.5cm]{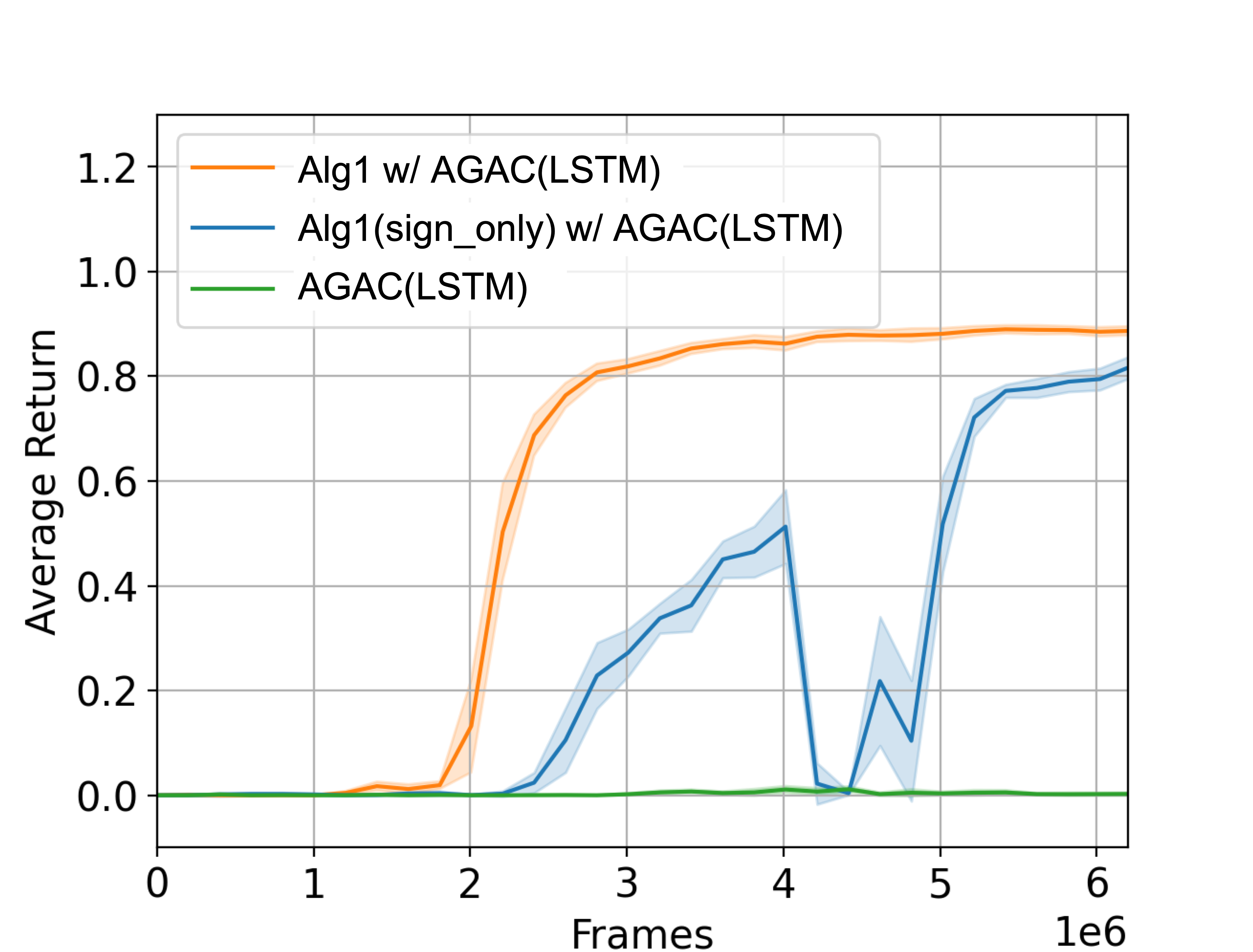}
         \caption{KeyCorridorS3R6}
         \label{fig6_8}
     \end{subfigure}    
 \caption{ Alg1 w/ AGAC/PPO indicates using AGAC or PPO as the policy learning algorithm in line 5 of Algorithm 1; Alg1(no c) w/ +PPO(CNN) indicates that running Algorithm 1 without considering its symbolic constraint while using PPO(CNN) in line 5; Alg1(sign\_only) w/ + AGAC(LSTM) indicates that running Algorithm 1 with only non-relational constraint while using AGAC(LSTM) in line 5; CNN and LSTM indicate the structures of the actor-critic networks.}
 \end{figure*}
\begin{figure*}
     \centering
         \begin{subfigure}[b]{0.32\textwidth}
         \centering
         \includegraphics[height=4.1cm, width=5.5cm]{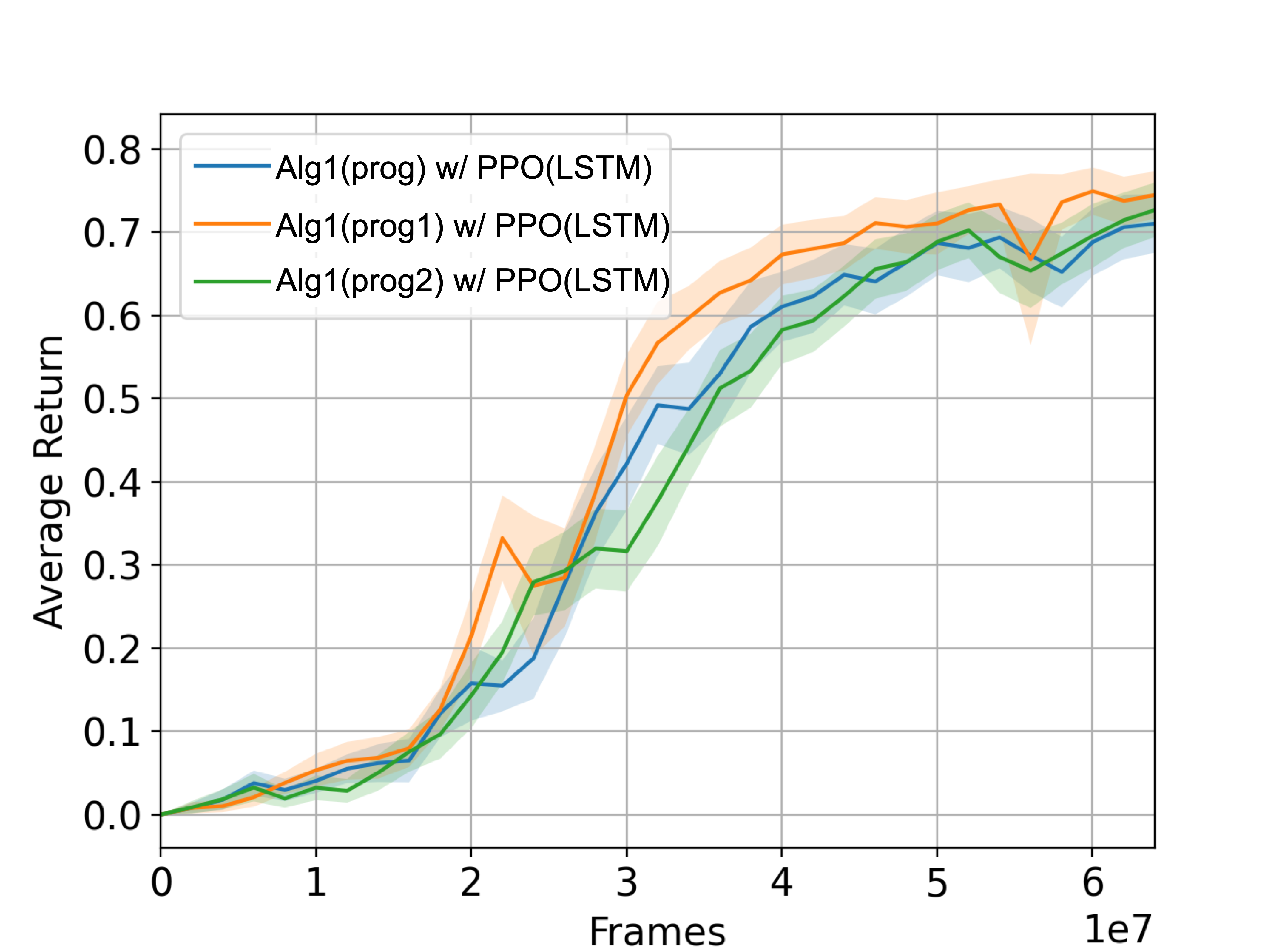}
         \caption{ObstructedMaze-2Dhlb}
         \label{fig6_6}
     \end{subfigure}
    \begin{subfigure}[b]{0.32\textwidth}
         \centering
         \includegraphics[height=4.1cm, width=5.5cm]{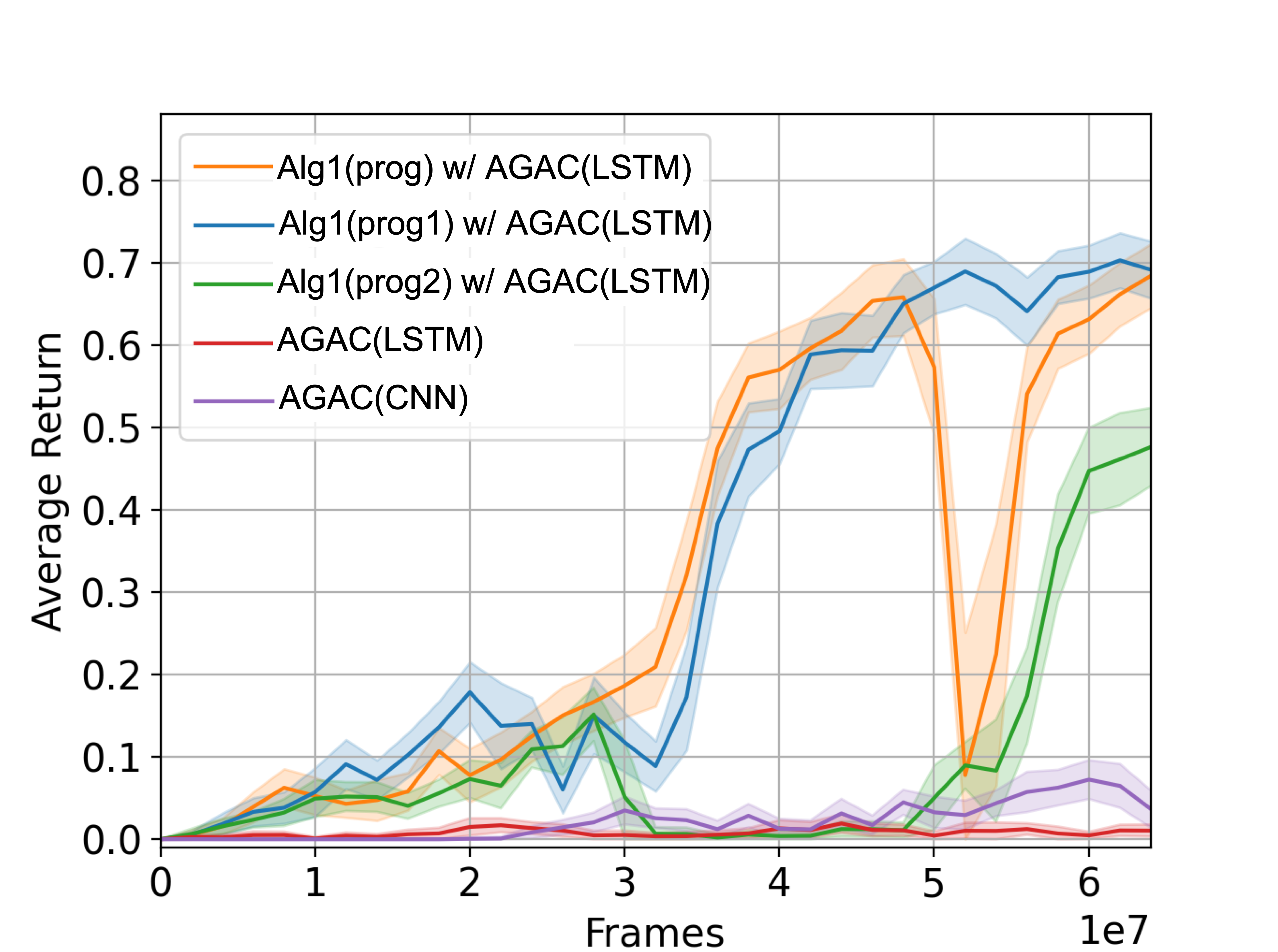}
         \caption{ObstructedMaze-2Dhlb}
         \label{fig6_7}
          \end{subfigure}
     \caption{ Alg1(prog/prog1/prog2) w/ AGAC/PPO indicate using AGAC or PPO as the policy learning algorithm in line 5 of Algorithm 1 while using different program sketches which are annotated with prog, prog1 and prog2. CNN and LSTM indicate the structures of the actor-critic networks.} 
\end{figure*}

\begin{figure*}
     \centering
    \begin{subfigure}[b]{0.32\textwidth}
         \centering
         \includegraphics[height=4.1cm, width=5.5cm]{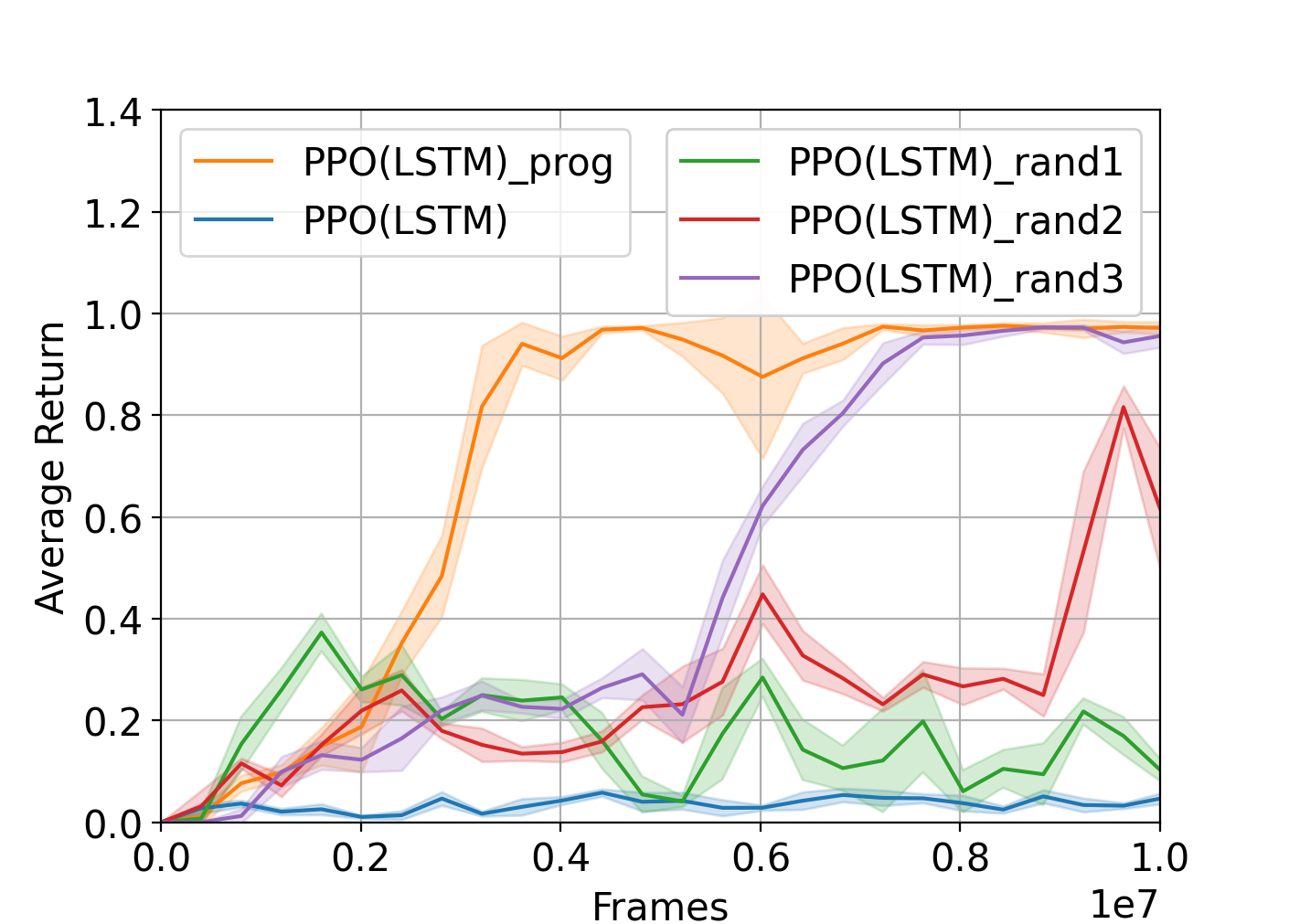}
         \caption{DoorKey-16x16}
         \label{fig6_4}
     \end{subfigure}     
     \begin{subfigure}[b]{0.32\textwidth}
         \centering
         \includegraphics[height=4.1cm, width=5.5cm]{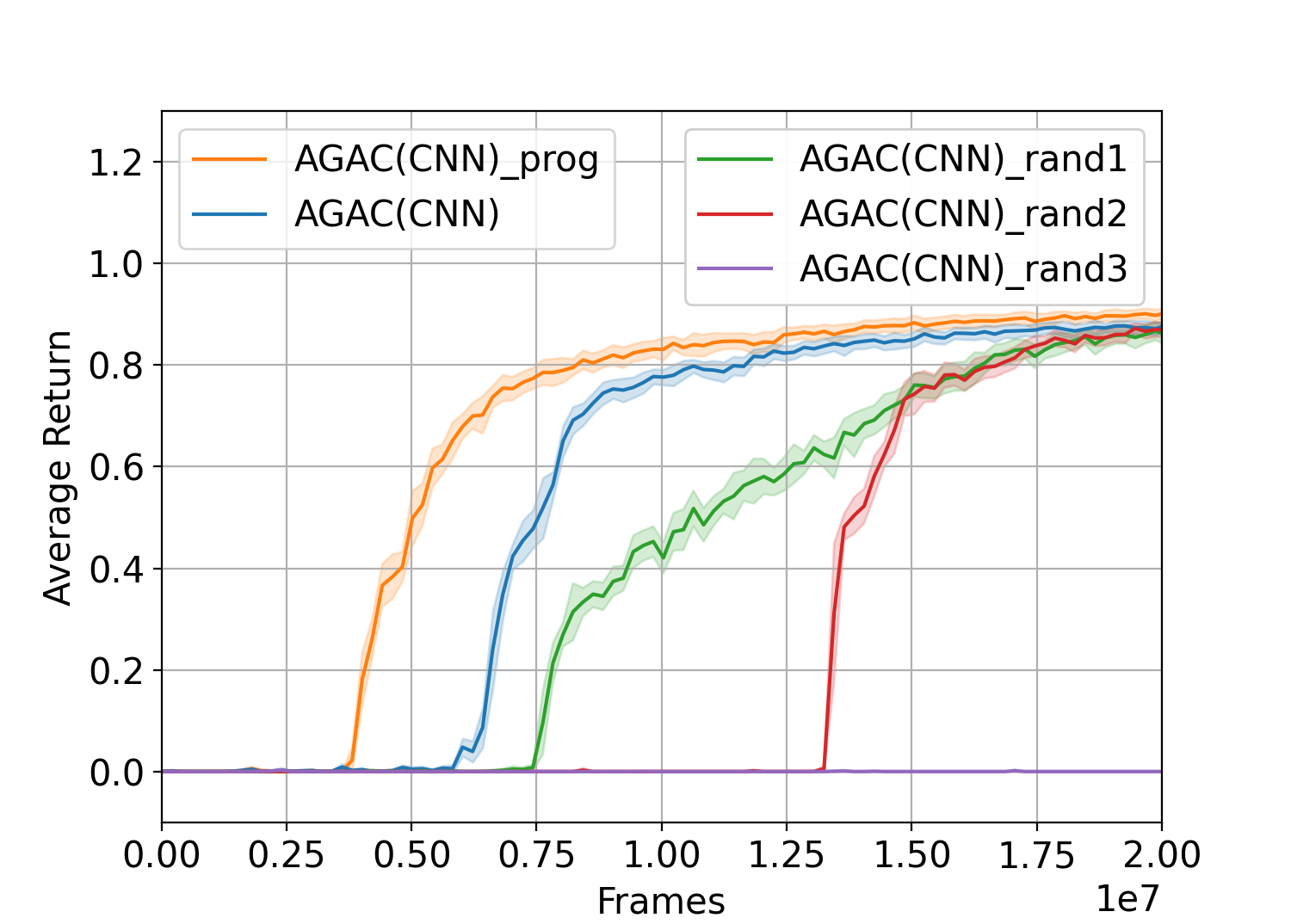}
        \caption{KeyCorridorS4R3} 
         \label{fig6_5}
     \end{subfigure}    
   \begin{subfigure}[b]{0.32\textwidth}
         \centering
         \includegraphics[height=4.1cm, width=5.5cm]{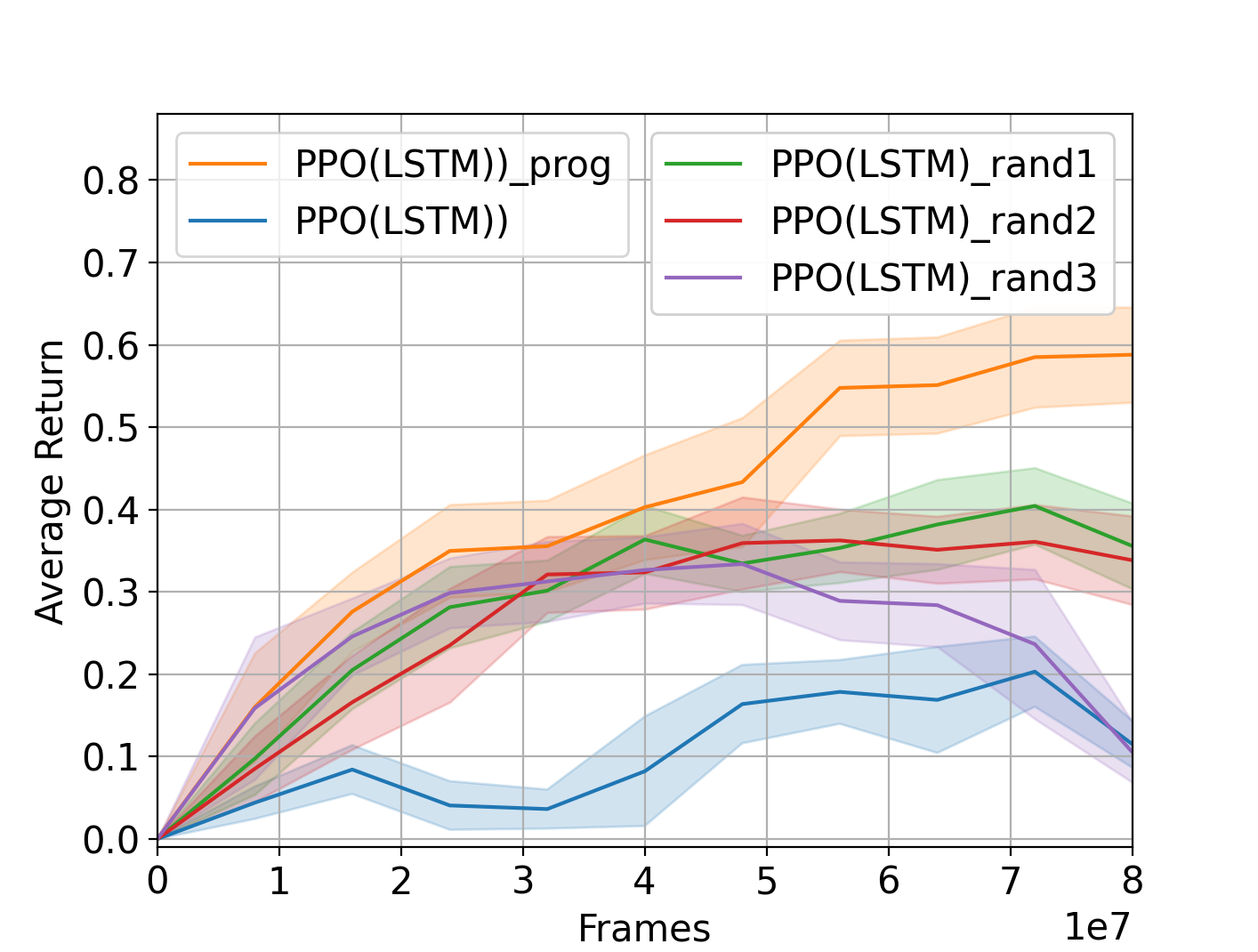}
        \caption{ObstructedMaze-Full} 
         \label{fig6_8}
     \end{subfigure}    
     \caption{In each task, given a sketch and a symbolic constraint, assign the holes with different values and use the induced programmatic reward functions to train agent policies. AGAC/PPO\_prog indicates that the hole assignments are learned via Algorithm 1; AGAC/PPO\_rand\# with an index \# indicates that the holes are randomly assigned with some values that satisfy the symbolic constraint for that task. CNN and LSTM indicate the structures of the actor-critic networks.} 
\end{figure*}

\subsection{Ablation Studies}
In addition to the questions investigated in the experimental section of the paper, we have performed ablation studies to investigate the following three questions. 
\begin{itemize}
    \item \textbf{E.} Is solving a PRD problem with random assignments to the holes sufficient to obtain an effective programmatic reward function for RL training?
    \item \textbf{F.} Can Algorithm 1 still work with weaker symbolic constraints or even without any symbolic constraint?
    \item \textbf{G.} Can Algorithm 1 work with different sketches?
\end{itemize}

For question \textbf{E}, we solve PRD problems for all three tasks by randomly generating hole assignments that satisfy the symbolic constraints. Those assignments are generated by only optimizing the supervised objective $J_c$ which is mentioned in the PRDBE section. 
\begin{itemize}
\item \textbf{DoorKey-16x16 }. In Fig.\ref{fig6_4}, we test three 3 randomly generated hole assignments for the programmatic reward function, each annotated by PPO(LSTM)\_rand\# . The trained PPO(LSTM) agents achieve higher performance with less amount of frames than the PPO(LSTM) agent that is trained with the default reward function. However, the programmatic reward function with a learned hole assignment, annotated by PPO(LSTM)\_prog, enables the agent to attain high performance with the lowest amount of frames. 
\item \textbf{KeyCorridorS4R4 }. we test 3 randomly generated hole assignments for  the programmatic reward functions, each  annotated by AGAC(CNN)\_rand\#.   As in Fig.\ref{fig6_5}, the agents trained with the programmatic reward functions with random assignments used even more frames than that trained with the default reward. In contrast, the agent trained with the programmatic reward function with a learned hole assignment achieves the best efficiency in achieving high performance. 
\item \textbf{Obstructed-Full }. In Fig.\ref{fig6_7}, we test 3 randomly assigned reward functions, each annotated by PPO(LSTM)\_rand\#. The training efficiency is competitive in the early stage of the training process. However, as the number of frames increases, the performance of the trained policies tend to decrease, which is unlike that of the programmatic reward function with learned hole assignment, as annotated by PPO(LSTM)\_prog. Note that the rewards sent to the PPO(LSTM) agents in this environment are all discounted by state-visitation count~\cite{flet-berliac2021adversarially}.
\end{itemize}

For question \textbf{F}, we conduct experiments on two of the tasks. 
\begin{itemize}
    \item \textbf{DoorKey-8x8 }. We test whether Algorithm 1 can properly train a policy even if \textbf{no symbolic constraint is provided}. That means, Algorithm 1 has to infer proper assignment for the holes solely by learning from demonstrations. As the plot annotated by Alg1(no c) w/ PPO(CNN) in Fig.\ref{fig6_3} shows, Algorithm 1 succeeded in training a PPO(CNN) policy to achieve the same level of performance with almost the same amount of frames as that PPO(CNN) policy trained with the default reward function, although both are inferior compared to running Algorithm 1 with the symbolic constraint included, this result is a strong evidence that validates Algorithm 1.  
    
    \item \textbf{KeyCorridorS3R3}. We modify the symbolic constraint by replacing the predicates concerning the relations between holes, such as the ones in Table.\ref{tab1}, with non-relational ones such as $\mathtt{?_{id}}\leq 0$ (detail will be shown in Table.\ref{tab6_1}-\ref{tab6_1_} and  explained after that). Basically, the non-relational constraint is weaker than its relational counterpart. The training result obtained with the non-relational constraint is shown by the curve annotated by Alg1(sign\_only) w/ AGAC(LSTM) in Fig.\ref{fig6_8}. By including the relational predicates, we obtain the curve annotated by Alg1 w/ AGAC(LSTM). Observe that by only adopting the non-relational predicates, the learning algorithm can still produce a policy with less frames than the one trained with the default reward, as annotated by AGAC(LSTM).  However, the training process is less stable compared with its relational counterpart.
\end{itemize}

For question \textbf{G}, beside the sketch that leads to the results in the \textbf{Experiment} section in the main text, we design two more sketches for the ObstructedMaze task. Those two sketches are annotated as prog1 and prog2 while the sketch involved in the \textbf{Experiment} section is annotated as prog. 
We run Algorithm 1 with these two sketches in the ObstructedMaze-2Dhlb environment and compare the results with the sketch annotated by prog in the Experiments section. Fig.\ref{fig6_6} shows the results of using PPO(LSTM) in line 5 of Algorithm 1 and Fig.\ref{fig6_7} shows the result of using AGAC(LSTM). We will describe the difference between these sketches in the appendix.

%% file: arxiv/concl.tex
\section{Conclusion}
We propose a novel paradigm for using {programs to specify the reward functions in RL environments.} We have developed a framework to complete a reward program sketch by learning from expert demonstrations. We experimentally validate our approach on challenging benchmarks and by comparing with SOTA baselines. Our future work will focus on reducing human efforts in the sketch creation process.

%% file: arxiv/sec6.tex
\clearpage
\section{Appendix}
In this appendix, we will present the sketches and symbolic constraints used in the experiments, experimental setup, and the alternative sampling scheme mentioned in the paper.

\subsection{Sketches and Symbolic Constraints}
In this section, we provide pseudo-code of the sketches as well as details on the symbolic constraints for the three tasks considered in the Experiments section. For readability purposes, we use the notation $\mathtt{traj}$ instead of the term $x$ specified in the syntax rule \eqref{syn3_1} to refer to the input trajectory.

\noindent \textbf{DoorKey}. We show the DoorKey environments of different sizes in Fig.\ref{fig6_0_1} and \ref{fig6_0_2}. The sketch is shown in Fig.\ref{fig6_0}. Most statements in the Fig.\ref{fig6_0} are self-explanatory. We use the function $\mathtt{pred}$ to identify the token of any given state-action tuple, then use $\mathtt{match}$ to compare the token of the given state-action tuple with the tokens listed in line 4, 5, 6, 9 and 12 to determine the reward. Line 10 and 13 both scan past states in the trajectory with function $\mathtt{filter}$ and use $\mathtt{len}$ to check whether the agent has unlocked the door before. In line 7, the conditional statement specifies that the agent will be penalized with $\mathtt{?_{3}}$ for closing the door, which is a redundant behavior for the task, unless the accumulated penalty will exceed the reward $\mathtt{?_2}$ for unlocking the door. This is a heuristic that prevents an under-trained agent from being excessively penalized for its actions related to the door. Note that by allowing putting holes in the guard of the conditional we make the programmatic reward functions non-linear in the holes (instead of being linear in the holes with one-hot coefficients). This is a major difference between programmatic reward functions and the linear functions used in the generic IRL methods. 

\begin{figure}
     \centering
     \begin{subfigure}[b]{0.23\textwidth}
         \centering
        \includegraphics[height=3.8cm,width=4cm]{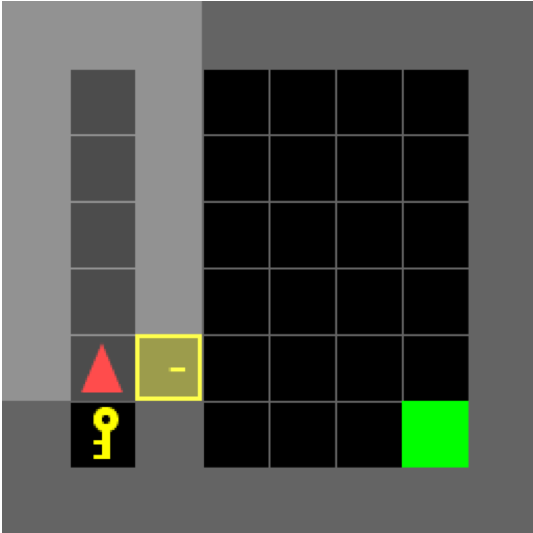}
        \caption{} 
         \label{fig6_0_1}
     \end{subfigure}
        \begin{subfigure}[b]{0.23\textwidth}
         \centering
         \includegraphics[height=3.8cm,width=4cm]{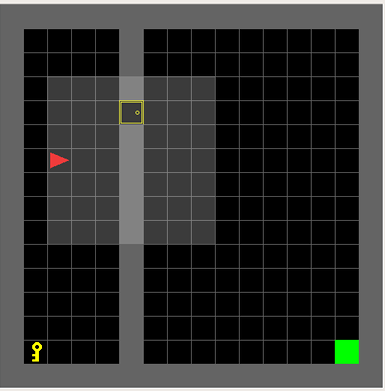}
         \caption{}
         \label{fig6_0_2}
     \end{subfigure}
     \begin{subfigure}[b]{0.44\textwidth}
     \includegraphics[height=5cm, width=10cm]{figures/fig6}
     \caption{} \label{fig6_0}
     \end{subfigure}
     \caption{(a)DoorKey-8x8; (b)DoorKey-16x16; (c)Our programmatic reward sketch for DoorKey task}
\end{figure}

\begin{table}[b]
    \centering
    \begin{tabular}{l|l}
            \textbf{Properties} & \textbf{Predicates} \\
            $[c_1]$Reward reaching the goal & $\mathtt{\bigwedge\limits^{5}_{\mathtt{id}=1}(\mathtt{?_{id}} \leq \mathtt{?_1})}$ \\
            $[c_2]$Penalize dropping unused key & $\mathtt{?_5+?_4\leq 0}$\\
            $[c_3]$Reward unlocking door & $\mathtt{\bigwedge\limits^5_{\mathtt{id}=2}({\mathtt{?_{id}}}\leq \mathtt{?_2})}$\\
            $[c_4]$Penalty for closing door & $\mathtt{\mathtt{?_3}\leq 0}$\\
            $[c_5]$Mildly penalize door toggling &$\mathtt{\mathtt{?_3}+\mathtt{?_2}\leq 0}$
    \end{tabular}
    \caption{The correspondence between properties and predicates for the DoorKey sketch in Fig.\ref{fig6_0}}
    \label{tab2}
\end{table}
  
The symbolic constraint for this task is defined as the conjunction of the predicates listed in Table.\ref{tab2}. In $c_1$ and $c_3$ we specify that the reward for unlocking the door is larger than any other behavior except for the reward for reaching the goal. In $c_2$, we specify that if the agent has not used the key to unlock the door, the total reward for picking up and dropping the key must be non-positive. In $c_4$ we specify that closing door should be penalized with non-positive reward while in $c_5$ we ensure that such penalty is limited.

\noindent \textbf{KeyCorridor}. We show the KeyCorridor environments of different sizes in Fig.\ref{fig6_1_1}-\ref{fig6_1_3}. We show the pseudo-code for this task in Fig.\ref{fig6_1}.

\begin{figure*}
     \centering
     \begin{subfigure}[b]{0.32\textwidth}
         \centering
        \includegraphics[height=4.8cm,width=5cm]{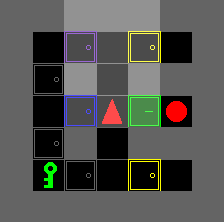}
        \caption{} 
         \label{fig6_1_1}
     \end{subfigure}
      \begin{subfigure}[b]{0.32\textwidth}
         \centering
         \includegraphics[height=4.8cm,width=5cm]{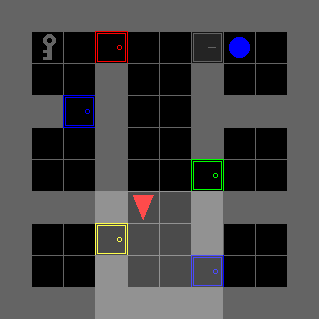}
         \caption{}
         \label{fig6_1_2}
     \end{subfigure} 
      \begin{subfigure}[b]{0.32\textwidth}
         \centering
         \includegraphics[height=4.8cm,width=5cm]{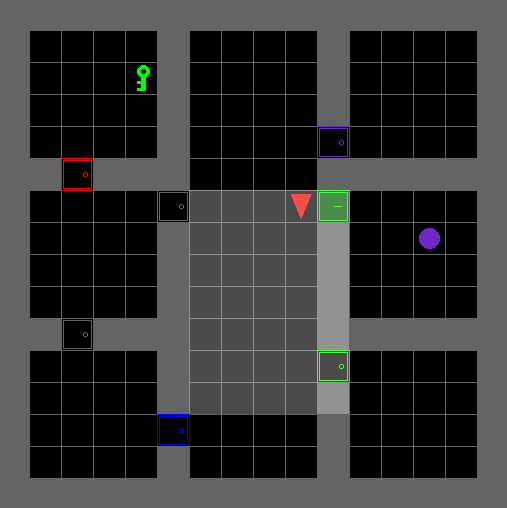}
         \caption{}
         \label{fig6_1_3}
     \end{subfigure} \hfill%
     \begin{subfigure}[b]{0.95\textwidth}
     \centering
     \includegraphics[height=9.cm, width=15cm]{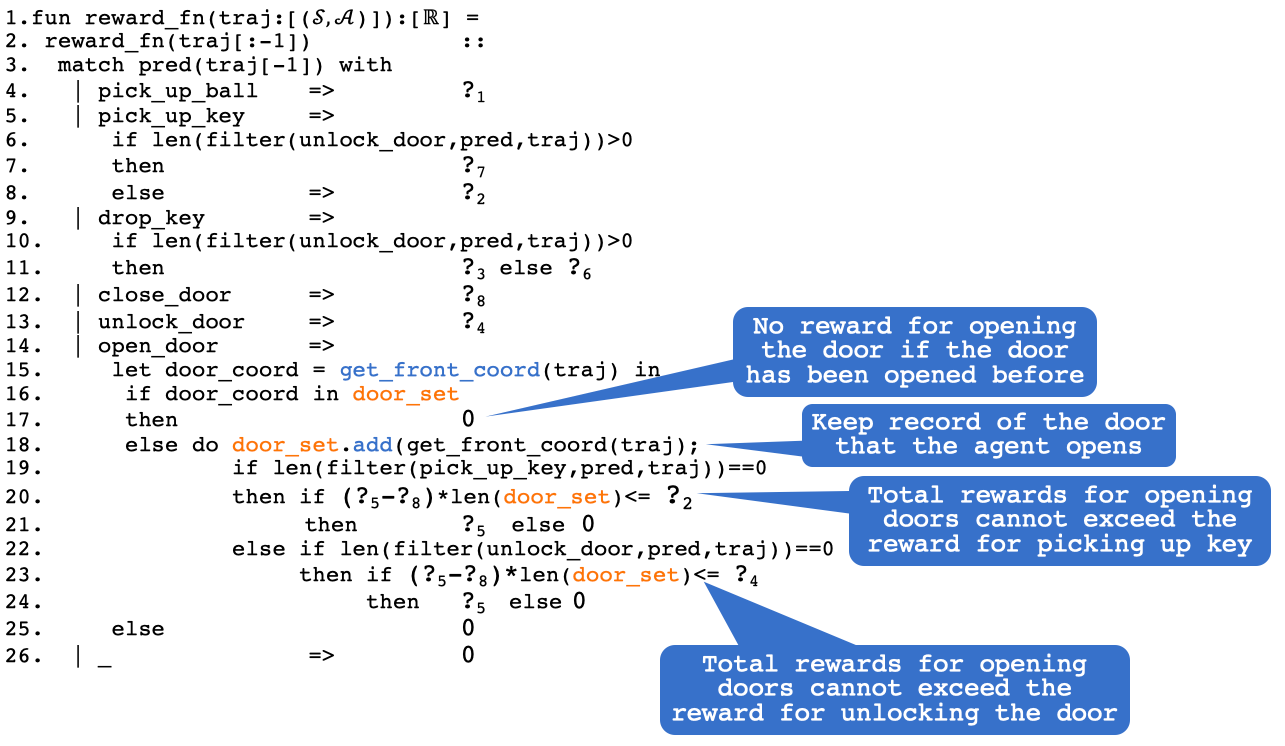}
     \caption{} \label{fig6_1}
     \end{subfigure}
     \caption{(a)KeyCorridorS3R3; (b)KeyCorridorS4R3; (c) KeyCorridorS6R3; (d) The pseudo-code of our program sketch for KeyCorridor task}
\end{figure*}

\begin{table}[t]
        \centering
        \begin{tabular}{l|l}
            \textbf{Properties} & \textbf{Predicates} \\
            $[c_1]$Reward picking up ball & $\mathtt{\bigwedge\limits^8_{id=2}(\mathtt{?_{id}}\leq \mathtt{?_1})}$\\
            $[c_2]$Reward 1st time picking up key & $\mathtt{ \mathtt{?_2}\geq 0}$\\
            $[c_3]$Reward dropping used key & $\mathtt{ \mathtt{?_3}\geq 0}$\\
            $[c_4]$Reward unlocking door & $\mathtt{ \mathtt{?_4}\geq 0}$ \\
            $[c_5]$Encourage opening door & $\mathtt{ \mathtt{?_5}\geq 0}$\\
            $[c_7]$Penalize meaningless move & $\mathtt{\mathtt{?_8}\leq0}$\\
            $[c_8]$Moderately reward opening door & $\mathtt{\mathtt{?_5}-\mathtt{?_8}\leq \mathtt{?_2}}$\\
            $[c_9]$Penalize dropping unused key &$\mathtt{\mathtt{?_2}+\mathtt{?_6}\leq 0}$\\
            $[c_{10}]$Penalize picking up used key &$\mathtt{\mathtt{?_3}+\mathtt{?_7}\leq 0}$\\
        \end{tabular}
        \caption{The correspondence between properties and predicates for the reward sketch of KeyCorridor in Fig.\ref{fig6_1}}
        \label{tab6_1}
    \end{table}

\begin{table}[t]
        \centering
        \begin{tabular}{l|l}
            \textbf{Properties} & \textbf{Predicates} \\
            $[c_1]$Reward picking up ball & $\mathtt{?_1}\geq 0$\\
            $[c_2]$Reward 1st time picking up key & $\mathtt{ \mathtt{?_2}\geq 0}$\\
            $[c_3]$Reward dropping used key & $\mathtt{ \mathtt{?_3}\geq 0}$\\
            $[c_4]$Reward unlocking door & $\mathtt{ \mathtt{?_4}\geq 0}$ \\
            $[c_5]$Encourage opening door & $\mathtt{ \mathtt{?_5}\geq 0}$\\
            $[c_7]$Penalize meaningless move & $\mathtt{\mathtt{?_8}\leq0}$\\
            $[c_8]$ Penalize dropping unused key &$\mathtt{?_6}\leq 0$\\
            $[c_{10}]$Penalize picking up used key &$\mathtt{?_7\leq 0}$\\
        \end{tabular}
        \caption{The non-relational predicates for the reward sketch of KeyCorridor in Fig.\ref{fig6_1}}
        \label{tab6_1_}
    \end{table}
 
This sketch maintains a set of door coordinates, $\mathtt{door\_set}$. When the agent opens a door, the sketch checks $\mathtt{door\_set}$ as in line 16 whether this door has been opened before. The function $\mathtt{get\_front\_coord}$ scans the past states in the trajectory and measures the relative position of the door w.r.t the agent's initial position. If the door has not been opened, the relative coordinate of the door is added to $\mathtt{door\_set}$ as in line 18. The sketch determines the reward for an opening-door behavior depending on whether the agent has found the key as in line 19. We identify this condition for the consideration that whereas the agent may have to search from door to door to find the key, it may not have to do another round of exhaustive search for the locked door afterwards, since it may have spotted the locked door before it finds the key. Especially, we implement a reward scheme in line 20 and 23 such that, if the accumulated rewards for opening doors have exceeded the rewards for finding the key or unlocking the door, which are both crucial milestones for finishing the task, the sketch outputs a reward $0$ instead of $\mathtt{?_5}$. Note that in line 20 and line 23 we subtract $\mathtt{?_8}$ to make sure that, given any trajectory, even if all the opening-door actions were replaced with some meaningless action, such as closing door, the agent can still gain even higher total reward by picking up a key or unlocking a door. This design heuristic aims at rewarding the agent to explore the environment while preventing the agent from excessive exploration. 

A symbolic constraint $c$ for the sketch is defined as the conjunction of the predicates listed in Table.\ref{tab6_1}. Note that when an assignment for the holes $\mathtt{{?}_{id=1:8}}$ is given, the maximum number of doors that the agent is encouraged to open is fixed due to the conditionals in line 19 and line 22 in Fig.\ref{fig6_1}. This feature allows the learned program to generalize to larger environments if the number of doors does not change too much. As mentioned earlier in the ablation study, we also tested a weaker symbolic constraint that only contains non-relational predicates. This symbolic constraint is shown in Table.\ref{tab6_1_}. It only specifies the signs of the holes and ignores the relational order between the holes.

\noindent\textbf{ObstructedMaze}.  We show the environments of different sizes in Fig.\ref{fig6_2_1} and \ref{fig6_2_2}. We designed three sketches for this task each implementing a different reward scheme. We show the pseudo-code for the first sketch, with which we obtain the results in the Experiments section as well as the ablation study for question \textbf{E}. We will also briefly illustrate the two other sketches. 

\begin{figure*}
     \centering
     \begin{subfigure}[b]{0.2\textwidth}
         \centering
        \includegraphics[height=3.8cm,width=1.5cm]{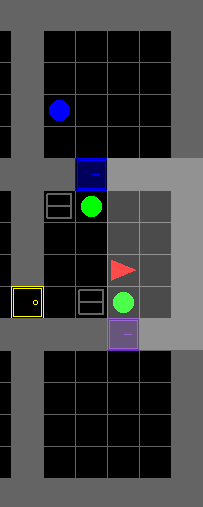}
        \caption{} 
         \label{fig6_2_1}
     \end{subfigure}
        \begin{subfigure}[b]{0.36\textwidth}
         \centering
         \includegraphics[height=3.8cm,width=4cm]{figures/fig9}
         \caption{}
         \label{fig6_2_2}
     \end{subfigure}\hfill%
     \begin{subfigure}[b]{0.95\textwidth}
     \centering
     \includegraphics[height=12cm, width=17cm]{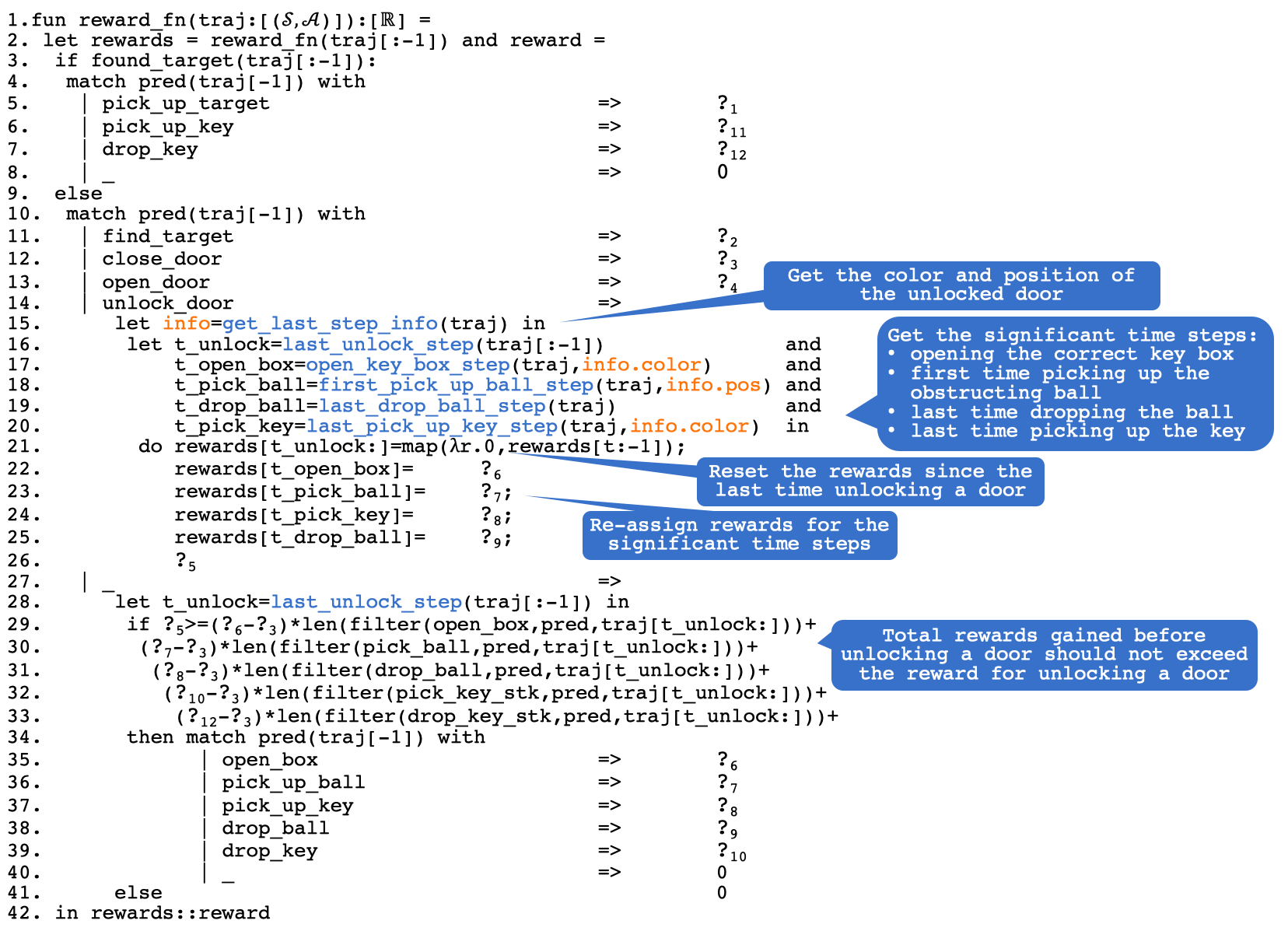}
     \caption{} \label{fig6_2}
     \end{subfigure}
     \caption{(a) ObstructedMaze-2Dhlb; (b) ObstructedMaze-Full; (c) Program sketch for ObstructedMaze task}
\end{figure*}

\begin{table}
        \centering
        \begin{tabular}{l|l}
            \textbf{Properties} & \textbf{Predicates} \\
            $[c_1]$Reward picking up target & $\mathtt{\bigwedge\limits^{12}_{id=2}(\mathtt{?_{id}}\leq \mathtt{?_1})}$\\
            $[c_2]$Reward finding target& $\mathtt{ \mathtt{?_2}\geq ?_4+?_5-2 ?_3}$\\
            $[c_3]$Reward opening door & $\mathtt{ \mathtt{?_4}\geq 0}$\\
            $[c_4]$Reward unlocking door & $\mathtt{\mathtt{?_{5}}\geq \sum\limits^{10}_{id=6}\mathtt{?_{id}}}$ \\
            $[c_5]$Penalize meaningless move & $\mathtt{\mathtt{?_3}\leq0}$\\
            $[c_6]$Penalize picking up used key &$\mathtt{?_{11}+?_{12}}\leq 0$\\
            $[c_7]$Reward opening box & $\mathtt{\mathtt{?_6}\geq0}$\\
            $[c_8]$Reward picking up ball & $\mathtt{\mathtt{?_7}\geq 0}$\\
            $[c_9]$Reward picking up key &$\mathtt{\mathtt{?_8}\geq 0}$\\
            $[c_{10}]$Reward dropping ball &$\mathtt{\mathtt{?_{8}}\geq 0}$\\
            $[c_{11}]$Reward dropping used key & $\mathtt{?_{12}}\geq 0$
        \end{tabular}
        \caption{The correspondence between properties and predicates for the reward sketch of ObstructedMaze task in Fig.\ref{fig6_2}}
        \label{tab6_2}
    \end{table}

From line 3 to 8 in Fig.\ref{fig6_2}, we specify that if the agent has spotted the targeted blue ball, the program will only respond to three behaviors of the agent: picking up the target in line 5, picking a key in line 6 and dropping the key in line 7. For all other behaviors the program will return a $0$ reward as in line 8. 
If the agent has not spotted the target, the program will bypass line 3-8 and execute line 10. We use the conditional statements from line 28-40 to implement a similar heuristic as that in line 19-24 of the sketch in Fig.\ref{fig6_1}, that is, to conditionally encourage exploration. 

From line 14-27, we implement a new heuristic for this task. If the program detects in line 16 that agent manages to unlock a door, the program checks the last time step when the agent unlocked a door as in line 16 and resets all the rewards from then till the contemporary step to $0$ as in line 21. Then from line 17-25 the program selectively re-assign 
rewards to some significant time steps. The functions called in line 15-20 all check the hindsight trajectory in the similar way as the function $\mathtt{get\_front\_coord}$ in Fig.\ref{fig6_1} does. We omit the details of implementations in those functions here. The underlying idea behind line 15-25 is that once the agent manages to unlock a door, the program is able to recognize which past behaviors directly contribute to this door unlocking outcome, e.g. by opening which box the agent found the key for this door, which green ball was obstructing this door, where the agent put the ball so that it would not longer obstruct the door. We adopt such a heuristic because it is cumbersome to judge every behavior of the agent in this task before observing any meaningful outcome. By only recognizing the milestone behaviors instead of carrying out a detailed motion planning, we make a trade-off between effectiveness and complexity in the in the programmatic reward function design. We remark that through line 21-25, the sketch modifies hindsight rewards based on the current information. Similar ideas have been proposed in \textit{hindsight experience replay}. However, the existing works define the rewards as the Euclidean distances between the agent and some targets, which do not suit our task. By conducting the programmatic reward function design procedures, we have the flexibility to adapt to different tasks. 

The symbolic constraint $c$ is defined as the conjunction of the predicates listed in Table.\ref{tab6_2}. Especially, $c_4$ together with the conditional from line 28-33 we make sure that the total reward gained from exploration is no larger than that gained from unlocking doors.

For the other two sketches, we make some modification on the sketch in Fig.\ref{fig6_2}. The first sketch, which we annotate as prog1, is different from the one in Fig.\ref{fig6_2} by removing lines 15-25. The second sketch, which we annotate as prog2, is different from the one in Fig.\ref{fig6_2} by additional treating opening door and unlocking door as exploration behaviors before finding the goal. Basically, when the agent is opening or unlocking a door as in line 13 or 14 in Fig.\ref{fig6_2}, prog2 subtracts $\mathtt{?_2}$ with the total rewards that the agent has gained from opening and unlocking doors in the trajectory.
If this subtraction ends up non-positive, prog2 no longer rewards the agent for opening or unlocking doors in the future. Furthermore, the behaviors listed from lines 35-39 in Fig.\ref{fig6_2} will no longer be rewarded either, which is equivalent to letting $\mathtt{?_5}$ equal $0$ in lines 29, thus having the conditional always points to line 41 henceforth. This modification basically makes sure that the total reward of the trajectory is upper-bounded by a value dependent on the reward $\mathtt{?_2}$.

\subsection{Training details}
\begin{itemize}
\item \textbf{Training Overhead}. We note that the sketches in Fig.\ref{fig6_0}, \ref{fig6_1} and \ref{fig6_2} all require checking hindsight experiences, or maintaining memory or other expensive procedures. However, line 7 of Algorithm 1 requires running all $K$ candidate programs on all $m$ sampled trajectories, which may incur a substantial overhead during training. Our solution is that, before sampling any program as in line 6 of Algorithm 1, we evaluate the result of $[\![e]\!](\tau_{A,i})$ for all the $m$ trajectories. As mentioned earlier, each $[\![e]\!](\tau_{A,i})$ is a partial program with the holes $\mathtt{\textbf{?}_e}$ being the free variable. By doing this, we only need to execute once for certain expensive procedures that do not involve the holes, such as running the function $len(filter())$ in all sketches, the function $\mathtt{get\_front\_coord}$ in Fig.\ref{fig6_1} and all the functions called in line 15-20 of Fig.\ref{fig6_2}. Then we use $q_\varphi$ to sample $K$ hole assignments $\mathtt{\{\textbf{h}_{k}\}}^K_{k=1}$ from $\mathcal{H}$ and feed them to $\{[\![e]\!](\tau_{A,i})\}^m_{i=1}$ to obtain $\{\{[\![l_k]\!](\tau_{A,i}\}^m_{i=1}\}^K_{k=1}$. By replacing line 6 and line 7 with those two steps in Algorithm 1, we significantly reduce the overhead. 
\item \textbf{Supervised Learning Loss}. In Algorithm 1, a supervised learning objective $J_c$ is used to penalize any sampled hole assignment for not satisfying the symbolic constraint. In practice, since our sampler $q_\varphi$ directly outputs the mean and log-variance of a multivariate Gaussian distribution for the candidate hole assignments, we directly evaluate the satisfaction of the mean. Besides, as mentioned earlier, in our experiments we only consider symbolic constraint as a conjunction of atomic predicates, e.g. $c=\wedge^n_{i=1}\mu_i$ and each $\mu_i$ only concerns linear relation between the holes, we reformulated each $\mu_i$ into a form $\lambda \textbf{h}.u_i(\textbf{h})\leq 0$ where $u_i$ is some linear function of the holes. We make sure that $(u_i(\textbf{h})\leq 0) \leftrightarrow (\mu_i(\textbf{h})==\top)$. Given a hole assignment $\textbf{h}$ output by a $q_\varphi$, we first calculate each $u_i(\textbf{h})$, which is now a real number, then we let  $J_c(q_\varphi)$ be a negative binary cross-entropy loss for $Sigmoid(ReLU([u_i(\textbf{h}), \ldots, u_n(\textbf{h})]^T)))$ with $0$ being the ground truth. This loss penalizes any $\textbf{h}$ that makes $u_i(\textbf{h})>0$. In this way $J_c(q_\varphi)$ is differentiable w.r.t $\varphi$. Thus, we do not implement the logarithmic trick when optimizing $J_c$. 

\noindent\textbf{Network Architectures}. Algorithm 1 involves an agent policy $\pi_\varphi$, a neural reward function $f_\theta$ and a sampler $q_\phi$. Each of the three is composed of one or more neural networks.
\begin{itemize}
    \item \textbf{Agent policy $\pi_\varphi$}. Depending on the tasks, we prepare two versions of actor-critic networks, a CNN version and an LSTM version. For the CNN version, we directly adopt the actor-critic network from the off-the-shelf implementation of AGAC~\cite{flet-berliac2021adversarially}. The CNN version has 3 convolutional layers each with 32 filters, 3$\times$3 kernel size, and a stride of 2. A diagram of the CNN layers can be found in~\cite{flet-berliac2021adversarially}. For the LSTM version, we simply concatenate 3 convolutional layers, which are the same as those in the CNN version, with a LSTM cell of which the state vector has a size of 32. The LSTM cell is then followed by multiple fully connected layers each to simulate the policy, value and advantage functions. The AGAC and PPO policies always share the identically structured actor-critic networks, in both the CNN and LSTM versions. While AGAC contains other components~\cite{flet-berliac2021adversarially}, the PPO agent solely consists of the actor-critic networks. 
    \item \textbf{Neural reward function $f_\theta$}. For all the tasks, we use identical networks. Each network has 3 convolutional layers each with 16, 32 and 64 filters, 2$\times$2 kernel size and a stride of 1. The last convolutional layer is concatenated with an LSTM cell of which the state vector has a size of 128. The LSTM cell is then followed by a 3-layer fully connected network where each hidden layer is of size 64. Between each hidden layer we use two $tanh$ functions and one Sigmoid function as the activation functions. The output of the Sigmoid function is the logit for each action in the action space $\mathcal{A}$. Finally, given an action in a state, we use softmax and a Categorical distribution to output the log-likelihood for the given action as the reward.
    \item \textbf{Sampler $q_\phi$}. Since the holes in our sketches all take numerical values. We implement for each sketch a sampler that outputs the mean and log-variance of a multivariate Gaussian distribution of which the dimension is determined by the number of holes in the sketch. The network structures, on the other hand, are identical across all tasks. The input to each sampler is a constant $[1, \ldots, 1]^T$ of size 20. Each sampler is a fully-connected network with 2 hidden layers of size 64. The activation functions are both $tanh$.  Suppose that there are $|\textbf{?}_e|$ holes in the sketch. Then the output of the sampler $q_\varphi$ is a vector of size no less than $2|\textbf{?}_e|$.  The $|\textbf{?}_e|$ most significant elements in the output vector will be used as the mean of the Gaussian, and the next $|\textbf{?}_e|$ most significant elements constitute a diagonal log-variance matrix.
    \end{itemize}
    \item{\textbf{Normalization}}. Besides outputting the mean and log-variance for the hole assignment, the sampler $q_\varphi$ additionally outputs a value $\log \hat{z}_l$ to normalize $\exp(l(\tau))$ in $J_{gen}$. Specifically, we introduce such normalization term because our formulated learning objective aims at having $\hat{p}(\tau|l)=p(\tau)\exp(l(\tau))$ match the probabilities $p_E(\tau)\approx p(\tau)\pi_E(\tau)$, which implies that $l(\tau)$ has to be negative such that $exp(l(\tau))\equiv \pi_E(\tau)$. However, negative $l(\tau)$ in our sketch design indicates penalization. During implementation, we replace every $[\![l]\!](\tau)[t]$ with $[\![l]\!](\tau)[t] - \log \hat{z}_l$ when calculating $J_{gen}$ and let $l(\tau):=\sum_t([\![l]\!](\tau)[t] - \log \hat{z}_l)$ in $\hat{p}(\tau|l)=p(\tau)\exp(l(\tau))$. Then by maximizing the ELBO in \eqref{eq4_11}, we on one hand search for the proper hole assignment for the sketch, and on the other hand search for a $\hat{z}_l$ such that $\hat{p}(\tau|l)\equiv p(\tau|l)\equiv p_E(\tau)$ can be possibly realized. Since $\hat{z}_l$ is constant, we still use $[\![l^*]\!](\tau)$ for the policy training in line 5 of Algorithm 1. Note that given such $\hat{z}_l$, the normalization term $Z_l$ in \eqref{eq4_9} still has to be introduced in case that the intermediately learned $\hat{z}_l$ does not accurately normalize $\hat{p}(\tau|l)$.  
    \item{\textbf{Hyperparameters}}. 
Most of the hyperparameters that appear in Algorithm 1 are summarized as in Table.\ref{tab0}. All hyperparameters relevant to AGAC are identical as those in~\cite{flet-berliac2021adversarially} although we do not present all of them in Table.\ref{tab0} in order to avoid confusion. The hyperparameter $\eta$ is made large to heavily penalize $q_\varphi$ when its output violates the symbolic constraint $c$.
\end{itemize}

\begin{table}
\centering
\begin{tabular}{l|l}
    \textbf{Parameter} & \textbf{Value} \\
    \# Epochs & 4 \\
    \# minibatches ($\pi_\varphi$) & 8\\
    \# batch size ($f_\theta$) & 32\\
    \# frames stacked (CNN $\pi_\varphi$) & 4 \\
    \# reccurence (LSTM $\pi_\varphi$) & 1\\
    \# recurrence ($f_\theta$) & 8\\
    Discount factor $\gamma$ & 0.99\\
    GAE parameter $\lambda$ & 0.95\\
    PPO clipping parameter $\epsilon$ & 0.2\\
    $K$ & 16\\
    $\alpha$ & 0.001\\
    $\beta$ & 0.0003\\
    $\eta$ & 1.e8\\
\end{tabular}
\caption{Hyperparameters used in the training processes}
\label{tab0}
\end{table}

\subsection{Alternative Sampling Scheme}
For any \zwcchange{entity}{term}\li{entity is a very general term} in the form of $J_{v_l}=\mathbb{E}_{\tau\sim p(\tau|l)}[v_l(\tau)]$ \li{as that?}\zwcchange{as that in \eqref{eq4_9} in the paper}{such as the $\mathbb{E}_{\tau_A\sim p(\tau_A|l)}[\cdot]$ part in \eqref{eq4_7},}  we can estimate it with $\hat{J}_{v_l}$ as in\li{using?} \eqref{eq5_3} with two batches of i.i.d trajectories $\{\tau_i\}^{m}_{i=1}$ and $\{\tau_j\}^{m}_{j=1}$ of $\pi_A$.  This scheme is equivalent to independently estimating $Z_l$ that appears in \eqref{eq4_9}. Assuming that $\overline{v}_l\in \underset{\tau:p(\tau)>0}{\max}\ |v_l(\tau)|$ is an upper-bound of $v_l$, we show in Theorem \ref{th5_0} that the chance of $\hat{J}_{v_l}$ falling in a bounded neighborhood of $J_{v_l}$ increases with $m$. For the theorem to hold, we require that ${{\pi}_A(\tau)}$ is positively lower-bounded if $p(\tau)>0$, which in practice can be realized by \zwcadd{assuming that $\mathcal{A}$ is bounded and letting $\pi_A(a|s)>\epsilon$ for any $s, a$ with some $\epsilon >0$.}\zwccross{letting $\pi_A$ always reserve certain constant probability of uniformly selecting actions in $\mathcal{A}$} \li{can we be a bit more precise than certain?}. 
\begin{eqnarray}
\quad\mathclap{\hat{J}_{v_l}}\ &:=&\frac{\sum\limits^m_{i=1}\frac{\exp(l(\tau_i))}{\pi_A(\tau_i)}v_l(\tau_i)}{\sum\limits^m_{j=1}\frac{\exp(l(\tau_j))}{\pi_A(\tau_j)}}\label{eq5_3}
\end{eqnarray}

\begin{theorem}\label{th5_0}
Given a program $l$, a bounded function $v_l(\tau)\in[-\overline{v}_l, \overline{v}_l]$ and a lower-bounded agent policy $\pi_A$, i.e. $\forall \tau.p(\tau)>0 \Rightarrow \pi_A(\tau)\geq \underline{\pi_A}$, for any $\gamma>0$, the probability of $\hat{J}_{v_l}-J_{vl}\in\big[\frac{\hat{Z}_l J_{v_l}-\gamma  }{\hat{Z}_l+\gamma /\overline{v}_l}, \frac{\hat{Z}_l J_{v_l}+\gamma  }{\hat{Z}_l-\gamma /\overline{v}_l}\big]$ is no less than $\Big(1 - \exp(\frac{-2m\gamma^2\underline{\pi_A}^2/\overline{v}_l^2}{ \underset{\tau:p(\tau)>0}{\max}\exp(2l(\tau))})\Big)^4$.
\end{theorem}
\begin{proof}
We first show in \eqref{eq5_16} and \eqref{eq5_15} that the numerator and denominator of  $\hat{J}_{c, l}$ are respectively unbiased estimates of $Z_l J_{vl}$ and $Z_l$. 
\begin{eqnarray}
&&\mathclap{\underset{\tau_j\sim \pi_A}{\mathbb{E}}}\quad \Big[\frac{\exp(l(\tau_j)) }{\pi_A(\tau_j)}\Big]=\quad\mathclap{\sum\limits_{\tau_j:\pi_A(\tau_i)>0}}\quad \frac{\pi_A(\tau_j)p(\tau_j)\exp(l(\tau_j))}{\pi_A(\tau_j)}\nonumber\\
&=&\quad\mathclap{\sum\limits_{\tau_j:\pi_A(\tau_i)>0}}\quad p(\tau_j)\exp(l(\tau_j)) =Z_l\label{eq5_16}\\
&&\mathclap{\underset{\tau_i\sim \pi_A}{\mathbb{E}}}\quad \Big[\frac{\exp(l({\tau_i}))v_l({\tau_i})}{\pi_A({\tau_i})}\Big]
=Z_l J_{v_l}\label{eq5_15}
\end{eqnarray}
By Hoeffding's inequality, for arbitrary $\gamma>0$ we have the confidence \eqref{eq5_2} and \eqref{eq5_2_1} respectively on the two batches of $m$ i.i.d sampled trajectories. The term $\underset{\tau:p(\tau)>0}{\max}{\frac{v_l{\ \ \ \mathclap{\exp}\ \ \ }(l)}{{\pi}_A}}(\tau)^2$ is an abbreviation of $\underset{\tau:p(\tau)>0}{\max}\big({\frac{{v_l(\tau)}{\ \ \ \mathclap{\exp}\ \ \ }(l(\tau))}{{\pi}_A(\tau)}}\big)^2$.
\begin{eqnarray}
&&\mathclap{\mathcal{P}}\ \Big(\frac{1}{m}\sum\limits^m_{i=1}\frac{\exp(l(\tau_i))v_l(\tau_i)}{\pi_A(\tau_i)}- Z_l J_{v_l}\geq -\gamma  \Big)\nonumber \\
&\geq& 1-  \ \ \ \mathclap{\exp}\ \ \ \Big(\frac{-2m\gamma^2  }{\underset{{\tau}}{\max}{\frac{v_l{\ \ \ \mathclap{\exp}\ \ \ }(l)}{{\pi}_A}}({\tau})^2}\Big)\nonumber\\
&\geq& 1- \ \ \ \mathclap{\exp}\ \ \ \Big(\frac{-2m\gamma^2/\overline{v}_l^2}{  \underset{\tau:p(\tau)>0}{\max}\frac{\exp(2l(\tau))}{{\pi}_A(\tau)^2}}\Big)\label{eq5_2}\\
&&\mathclap{\mathcal{P}}\ \Big(\frac{1}{m}\sum\limits^m_{j=1}\frac{\exp(l(\tau_j))}{\pi_A(\tau_j)}-Z_l\leq \frac{\gamma  }{\overline{v}_l}\Big) \nonumber\\
&\geq& 1 -  \ \ \ \mathclap{\exp}\ \ \ \Big(\frac{-2m\gamma^2/\overline{v}_l^2}{  \underset{\tau:p(\tau)>0}{\max}\frac{\exp(2l(\tau))}{{\pi}_A(\tau)^2}}\Big)\label{eq5_2_1}
\end{eqnarray} 
By conjoining \eqref{eq5_2} and \eqref{eq5_2_1}, the confidence on the lower-bound of $\hat{J}_{v_l}$ is lower-bounded by \eqref{eq5_4_1} for any $\gamma>0$. 
\begin{eqnarray}
&\mathclap{\mathcal{P}}\ \Big(\hat{J}_{v_l}\geq \frac{Z_l J_{v_l}-\gamma  }{Z_l+\gamma /\overline{v}_l}\Big)=  \ \ \mathclap{\mathcal{P}}\ \Big(\hat{J}_{v_l}\geq J_{v_l} - \frac{\gamma \overline{v}_l+\gamma J_{v_l}}{Z_l\overline{v}_l  + \gamma}\Big)&\nonumber\\
&\geq {\Big(}1 - \exp(\frac{-2m\gamma^2/\overline{v}_l^2}{ \underset{\tau:p(\tau)>0}{\max}\frac{\exp(2l(\tau))}{{\pi}_A(\tau)^2}})\Big)^2&\label{eq5_4_1} 
\end{eqnarray}
Again, by Hoeffding's inequality, for arbitrary $\gamma>0$ we have the confidence \eqref{eq5_5} and \eqref{eq5_5_1} on those two batches of $m$ i.i.d sampled trajectories.
\begin{eqnarray}
&&\mathclap{\mathcal{P}}\ \Big(\frac{1}{m}\sum\limits^m_{j=1}\frac{\exp(l(\tau_j))v_l(\tau_j)}{\pi_A(\tau_j)}- Z_l J_{v_l}\leq \gamma  \Big)\nonumber \\
&\geq& 1- \ \ \ \mathclap{\exp}\ \ \ \Big(\frac{-2m\gamma^2/\overline{v}_l^2}{  \underset{\tau:p(\tau)>0}{\max}\frac{\exp(2l(\tau))}{{\pi}_A(\tau)^2}}\Big)\label{eq5_5}\\
&&\mathclap{\mathcal{P}}\ \Big(\frac{1}{m}\sum\limits^m_{i=1}\frac{\exp(l(\tau_i))}{\pi_A(\tau_i)}-Z_l\geq -\frac{\gamma  }{\overline{v}_l}\Big) \nonumber\\
&\geq& 1 -  \ \ \ \mathclap{\exp}\ \ \ \Big(\frac{-2m\gamma^2/\overline{v}_l^2}{  \underset{\tau:p(\tau)>0}{\max}\frac{\exp(2l(\tau))}{{\pi}_A(\tau)^2}}\Big)\label{eq5_5_1}
\end{eqnarray} 
By conjoining \eqref{eq5_5} and \eqref{eq5_5_1}, the confidence on the upper-bound of $\hat{J}_{v_l}$ is upper-bounded by \eqref{eq5_6_1} for any $\gamma>0$. 
\begin{eqnarray}
&\mathclap{\mathcal{P}}\ \Big(\hat{J}_{v_l}\leq \frac{Z_l J_{v_l}+\gamma  }{Z_l-\gamma /\overline{v}_l}\Big)=  \ \ \mathclap{\mathcal{P}}\ \Big(\hat{J}_{v_l}\leq J_{v_l} + \frac{\gamma \overline{v}_l+\gamma J_{v_l}}{Z_l\overline{v}_l - \gamma}\Big)&\nonumber\\
&\geq {\Big(}1 - \exp(\frac{-2m\gamma^2/\overline{v}_l^2}{ \underset{\tau:p(\tau)>0}{\max}\frac{\exp(2l(\tau))}{{\pi}_A(\tau)^2}})\Big)^2&\label{eq5_6_1} 
\end{eqnarray}
Then by conjoining \eqref{eq5_4_1} and \eqref{eq5_6_1}, we further obtain a confidence on $\hat{J}_{v_l}$ being in the interval as shown in \eqref{eq5_7}.
\begin{eqnarray}
&&\mathclap{\mathcal{P}}\ \Big(\hat{J}_{v_l} - J_{v_l}\in\big[\frac{Z_l J_{v_l}-\gamma  }{Z_l+\gamma /\overline{v}_l}, \frac{Z_l J_{v_l}+\gamma  }{Z_l-\gamma /\overline{v}_l}\big]\Big)\nonumber\\
&\geq& {\Big(}1 - \exp(\frac{-2m\gamma^2/\overline{v}_l^2}{ \underset{\tau:p(\tau)>0}{\max}\frac{\exp(2l(\tau))}{{\pi}_A(\tau)^2}})\Big)^4\nonumber\\
&\geq&\Big(1 - \exp(\frac{-2m\gamma^2\underline{\pi_A}^2/\overline{v}_l^2}{ \underset{\tau:p(\tau)>0}{\max}\exp(2l(\tau))})\Big)^4\label{eq5_7}
\end{eqnarray}
\end{proof}